\documentclass{new_tlp}
\usepackage{mathptmx}

\DeclareMathAlphabet{\mathcal}{OMS}{cmsy}{m}{n}

\RequirePackage[latin1]{inputenc}
\usepackage[english]{babel}

\usepackage{graphicx}
\usepackage{amssymb}
\usepackage{amsmath}
\usepackage{stmaryrd}
\usepackage{bm}
\usepackage{enumitem}
\usepackage[colorinlistoftodos, textwidth=3cm]{todonotes} 
\usepackage[spaces,hyphens]{url}
\usepackage{xspace}
\usepackage{listings}
\newtheorem{definition} {Definition}
\newtheorem{lemma}      {Lemma}

\newtheorem{theorem}    {Theorem}
\newtheorem{example}{Example} 

\title[Reasoning on Multi-Relational Contextual Hierarchies via ASP with Algebraic Measures]{Reasoning 
on Multi-Relational Contextual Hierarchies via Answer Set Programming with Algebraic Measures}

  \author[L. Bozzato, T. Eiter, R. Kiesel]
         {LORIS BOZZATO\\
         Fondazione Bruno Kessler, Via Sommarive 18, 38123 Trento, Italy\\
         \email{bozzato@fbk.eu}
				 \and 
				 THOMAS EITER \qquad\quad RAFAEL KIESEL\\
				 Technische Universit\"{a}t Wien, Favoritenstra\ss e 9-11, A-1040 Vienna, Austria\\          				 \email{(thomas.eiter|rafael.kiesel)@tuwien.ac.at}
				}

\jdate{May 2021}
\pubyear{2021}
\pagerange{\pageref{firstpage}--\pageref{lastpage}}
\doi{S1471068401001193}

\def\isa{\sqsubseteq}
\def\CKB{\mathfrak{K}}

\def\I{\mathcal{I}}

\newcommand{\Pair}[2]{\left\langle#1,#2\right\rangle}
\newcommand{\Triple}[3]{\left\langle#1,#2,#3\right\rangle}

\newcommand{\vc}[1]{\mathbf{#1}}

\newcommand{\ee}{{\vc{e}}}

\newcommand{\ff}{{\vc f}}

\newcommand{\SROIQ}{\mathcal{SROIQ}}

\newcommand{\NI}{\mathrm{NI}}
\newcommand{\NR}{\mathrm{NR}}
\newcommand{\NC}{\mathrm{NC}}

\newcommand{\IC}{\mathfrak{I}}

\newcommand{\K}{\mathcal{K}}

\newcommand{\T}{\mathcal{T}}

\newcommand{\ov}[1]{\overline{#1}}

\newcommand{\stru}[1]{\langle #1 \rangle}

\newcommand{\non}{\neg}

\newcommand{\subs}{\sqsubseteq}

\newcommand{\Acal}{{\cal A}}

\newcommand{\Ecal}{{\cal E}}

\newcommand{\Ical}{{\cal I}}
\newcommand{\Lcal}{{\cal L}}

\newcommand{\Rcal}{{\cal R}}

\newcommand{\cov}[1]{\preceq}

\newcommand{\CAS}{\mathit{CAS}}

\newcommand{\OVR}{\mathit{OVR}}
\newcommand{\casmap}{\chi}
\newcommand{\mi}[1]{\mathit{#1}}

\newcommand{\nop}[1]{}

\newcommand{\SROIQrl}{\mathcal{SROIQ}\text{-RL}}

\newcommand{\KB}{\mathrm{K}} 

\newcommand{\mlc}{\ml{c}}

\newcommand{\smlc}{\sml{c}}

\newcommand{\subClass}{{\tt subClass}}
\newcommand{\ptop}{{\tt top}}
\newcommand{\subEx}{{\tt subEx}}
\newcommand{\subRole}{{\tt subRole}}

\newcommand{\supEx}{{\tt supEx}}
\newcommand{\pbot}{{\tt bot}}
\newcommand{\subConj}{{\tt subConj}}
\newcommand{\subRChain}{{\tt subRChain}}
\newcommand{\pDis}{{\tt dis}}
\newcommand{\pInv}{{\tt inv}}
\newcommand{\pIrr}{{\tt irr}}

\newcommand{\nom}{{\tt nom}}
\newcommand{\cls}{{\tt cls}}
\newcommand{\rol}{{\tt rol}}

\newcommand{\peq}{{\tt eq}}

\newcommand{\subEval}{{\tt subEval}}
\newcommand{\subEvalR}{{\tt subEvalR}}

\newcommand{\supForall}{{\tt supForall}}

\newcommand{\supLeqOne}{{\tt supLeqOne}}

\newcommand{\eval}{\textsl{eval}}

\newcommand{\N}{\boldsymbol{\mathsf{N}}}

\newcommand{\ml}[1]{\mathsf{#1}} 
\newcommand{\sml}[1]{{\mbox{\small\ensuremath{\ml{#1}}}}} 

\newcommand{\default}{{\mathrm D}}

\newcommand{\pcontext}{{\tt context}}
\newcommand{\prelation}{{\tt relation}}

\newcommand{\pprec}{{\tt prec}}
\newcommand{\ppreceq}{{\tt preceq}}
\newcommand{\ovr}{{\tt ovr}}
\newcommand{\naf}{\mathop{\tt not}}
\newcommand{\grd}{\mathit{grnd}}
\newcommand{\Head}{\mathit{H}}
\newcommand{\Body}{\mathit{B}}

\newcommand{\rif}{\leftarrow}

\newcommand{\insta}{{\tt insta}}
\newcommand{\instd}{{\tt instd}}
\newcommand{\triplea}{{\tt triplea}}
\newcommand{\tripled}{{\tt tripled}}

\newcommand{\unsat}{{\tt unsat}}
\newcommand{\test}{{\tt test}}
\newcommand{\testf}{{\tt test\_fails}}
\newcommand{\nlit}{{\tt nlit}}
\newcommand{\nrel}{{\tt nrel}}
\newcommand{\main}{\ml{main}}
\newcommand{\smlmain}{\sml{main}}

\newcommand{\definst}{{\tt def\_insta}}

\newcommand{\defsubs}{{\tt def\_subclass}}
\newcommand{\defsubcnj}{{\tt def\_subcnj}}
\newcommand{\defsubex}{{\tt def\_subex}}
\newcommand{\defsupex}{{\tt def\_supex}}
\newcommand{\defsupforall}{{\tt def\_supforall}}
\newcommand{\defsupleqone}{{\tt def\_supleqone}}
\newcommand{\defsubr}{{\tt def\_subr}}
\newcommand{\defsubrc}{{\tt def\_subrc}}
\newcommand{\defdis}{{\tt def\_dis}}
\newcommand{\definv}{{\tt def\_inv}}
\newcommand{\defirr}{{\tt def\_irr}}

\newcommand{\preceqex}{{\tt preceq\_except}}
\newcommand{\possovr}{{\tt p\_ovr}}
\newcommand{\indivi}{{\tt ind}}

\newcommand{\SROIQrld}{\mathcal{SROIQ}\text{-RLD}}

\def\qed{$\Box$}
\def\endproof{\ifhmode\nobreak\qed\par\fi\medskip}

\setitemize{label=--,leftmargin=*}
\setenumerate{leftmargin=*}
\newcommand{\EndEx}{\mbox{}~\hfill$\Diamond$} 

\newcommand{\comment}[1]{{#1}}

\newcommand{\local}{local}
\newcommand{\branch}{br}

\newcommand{\lf}[1]{\mbox{\texttt{#1}}}

\newcommand{\Cl}{\mathfrak{C}}
\newcommand{\prect}{\prec_t}
\newcommand{\precc}{\prec_c}

\def\Lcals{\Lcal_{\Sigma}}
\def\Lcalsn{\Lcal_{\Sigma,\N}}

\def\Ic{\I(\mlc)}
\def\Icp{\I(\mlc')}

\def\ICAS{\IC_{\CAS}}

\newcommand{\CKRew}{CKR\emph{ew}\xspace}

\newcommand{\DLlite}{\textsl{DL-Lite}}

\newcommand{\zero}{\ensuremath{\bm{0}}} 
\newcommand{\one}{\ensuremath{\bm{1}}}
\newcommand{\splus}{\ensuremath{\bm{+}}}
\newcommand{\stimes}{\ensuremath{*}}
\newcommand{\bplus}{\ensuremath{\textstyle\Sigma}}
\newcommand{\btimes}{\ensuremath{\textstyle\Pi}}

\newcommand{\srzero}{\ensuremath{e_{\oplus}}}
\newcommand{\srone}{\ensuremath{e_{\otimes}}}
\newcommand{\srsplus}{\ensuremath{\oplus}}
\newcommand{\srstimes}{\ensuremath{\otimes}}
\newcommand{\srbplus}{\ensuremath{{\textstyle\bigoplus}}}

\newcounter{myenumctr}

\newenvironment{myenumerate}{\begin{list}{ {\bf(\arabic{myenumctr})}\ }{\usecounter{myenumctr}
\setlength{\topsep}{0pt}
\setlength{\leftmargin}{0pt}
\setlength{\itemsep}{0pt}
\setlength{\parsep}{0.15\baselineskip}
\setlength{\itemindent}{1.35\labelwidth}}}
{\end{list}}

\begin{document}

\label{firstpage}

\maketitle

  \begin{abstract}
       Dealing with context dependent knowledge has
	led to different formalizations of the notion of context.
        Among them is the Contextualized Knowledge Repository (CKR) framework,
	which is 
        rooted in description logics but links on the reasoning side
        strongly to logic programs and Answer Set Programming (ASP) in
        particular.
	%
        The CKR framework caters for reasoning with defeasible axioms and exceptions
        in contexts, which was extended to knowledge inheritance across contexts
	in a coverage (specificity) hierarchy. However, the approach
        supports only this single type of contextual relation and the
        reasoning procedures work only for restricted hierarchies,
        due to non-trivial issues with model preference under
        exceptions. In this paper, we overcome these limitations
	and present a generalization of CKR hierarchies 
        to multiple contextual relations, 
	along with their interpretation of defeasible axioms and
        preference. To support reasoning, we use ASP with algebraic measures,
        which is a recent extension of ASP with weighted formulas over
        semirings that allows one to associate quantities with
        interpretations depending on the truth values of propositional
        atoms. Notably, we show that for a relevant fragment  of  CKR
        hierarchies with multiple contextual relations, 
        query answering can be realized with the popular asprin
        framework. The algebraic measures approach is more
        powerful and enables e.g.\ reasoning with epistemic queries
        over CKRs, which opens  interesting perspectives for
        the use of quantitative ASP extensions in other applications.
				
				Under consideration for acceptance in Theory and Practice of Logic Programming (TPLP).%
  \end{abstract}


\section{Introduction}

Representing and reasoning with context dependent knowledge is a
fundamental theme in AI, with proposals dating back to the works of
\citeN{mccarthy-notes-on-formalizing-context-1993}
and
\citeN{giunchiglia-serafini-mlsystems-ai2004}.
It has
gained increasing attention for the Semantic Web as knowledge
resources must be interpreted with contextual information from
their metadata. Several approaches for contextual reasoning, most
based on description logics, were developed~\cite{
stra-etal-2010,klarman:13,serafini-homola-ckr-jws-2012}.

A rich framework among them are \emph{Contextualized Knowledge
  Repositories (CKR)}
\cite{serafini-homola-ckr-jws-2012}:
CKR knowledge bases (KBs) are 2-layered structures with a
\emph{global context}, which contains context-independent
\emph{global} knowledge and \emph{meta-knowledge} about the structure
of the KB, and \emph{local contexts} containing knowledge about
specific situations (e.g., a region in space, a site of an
organization).  Notably, the global knowledge is propagated to local
contexts, where inherited axioms may be \emph{defeasible}, meaning
that instances can be ``overridden'' on an exceptional basis
\cite{BozzatoES:18}. Reasoning from CKRs strongly links to logic
programming, as the KBs are over
a Horn-description logic and the working of defeasible axioms was
inspired by conflict handling in \emph{inheritance logic
programs}~\cite{BuccafurriFL:99}. Furthermore, answering instance and conjunctive
queries over a CKR is possible via a uniform ASP program that employs
a materialization calculus akin to the one by \citeN{Krotzsch:10}.

For modeling and analyzing complex scenarios where global regulations
(e.g.\ laws, environmental regulations, access control rules)
can be refined by more specific situations (e.g.\ time-bounded events,
geographical areas, groups of users),
the CKR model
was extended \cite{DBLP:conf/kr/BozzatoSE18}
to cater for defeasible axioms in local contexts and knowledge
inheritance across hierarchies, based on a \emph{coverage} contextual
relation~\cite{serafini-homola-ckr-jws-2012}.

This approach, however, is limited to reason only on hierarchies based
on this single type of contextual relation.  In practice, defeasible
inheritance may be necessary under different contextual
relations. For example, along a location hierarchy, we may prefer
axioms encoding regional laws overriding state-level regulations,
while preferring newer rules over older laws along a temporal
dimension.  
%
A further limitation is that even for a single coverage relation, it is
challenging to encode the induced preference relation over CKR interpretation
using ASP because the relation may not be transitive and
thus not a strict partial order, as assumed e.g.\ in the popular asprin
framework for preferences in ASP \cite{DBLP:conf/aaai/BrewkaD0S15}.
Instead, a specialized implementation for preferential reasoning was
introduced~\cite{BozzatoES:CONTEXT19}, which however
needs to consider all answer sets of a program to single out a
preferred CKR model.


In this paper, we overcome these limitations and make the following
contributions:

\begin{myenumerate}
\item We generalize single-relational CKRs to multi-relational CKRs, where
axioms are not defeasible in general but merely with regard to individual
relations of hierarchies. By a combination of preferences over the distinct individual relations, 
we obtain an overall preference over the models of a CKR. While
intuitive, the technical condition has pitfalls and needs care. 

\item We show how to model 
multi-relation CKRs in ASP. Specifically, we use to this end 
ASP with \emph{algebraic measures} \cite{EiterK:20}, which is  a
foundation to express many quantitative reasoning problems. Here, \emph{weighted
logic} formulas~\cite{DBLP:conf/icalp/DrosteG05}
\emph{measure} values associated with an
interpretation $\mathcal{I}$ by performing a
computation over a semiring, whose outcome depends on
the truth of the propositional variables in
$\mathcal{I}$. Such measures can be used for e.g.\ weighted model
counting, probabilistic reasoning and, as in our case,
preferential reasoning.

\item While asprin is a powerful tool for modeling preferences in ASP,
  it appears to be ill-suited for expressing multi-relational CKR. The
  reason are \eval-expressions in CKRs, which propagate predicate
  extensions from one local context to another. We show, however, that
  under a well-behaved use of such expressions according to a
  syntactic disconnectedness condition, multi-relational CKRs can be
  expressed in asprin. This enables us to 
	\comment{use the asprin solver} to evaluate preferences for CKRs, which is showcased in a prototype
  implementation.

\item Furthermore, ASP with algebraic measures opens the possibility
  of reasoning tasks for CKRs beyond asprin's capability, even in
  absence of \eval-expression. As examples we consider obtaining
  preferred CKR models by overall weight queries and epistemic
  reasoning, which for description logics is specifically needed in 
	aggregate queries \cite{calvanese2008aggregate}.
\end{myenumerate}
In conclusion, ASP extended with preferences or algebraic
computations is a valuable tool to express CKR extensions and 
reasoning on them, with a promising perspective for further research.


\section{Preliminaries}
\label{sec:prelim}

\smallskip\noindent
\textbf{Description Logics and $\SROIQrl$.} 
%
We follow the common 
presentation of \emph{description logics (DLs)}~\cite{dlhb}
and the definition of the logic $\SROIQ$~\cite{horrocks:2006}.

A \emph{DL vocabulary} 
$\Sigma$ consists of the mutually disjoint countably infinite
sets $\NC$ of \emph{atomic concepts},
$\NR$ of \emph{atomic roles}, and 
$\NI$ of \emph{individual constants}.
%
Complex \emph{concepts} are recursively defined as the smallest
sets containing all concepts that can be inductively constructed using
the operators of the considered DL language $\Lcals$.
\comment{A DL \emph{knowledge base} $\K=\Triple{\T}{\Rcal}{\Acal}$ consists of: a TBox $\T$ which
can contain general concept inclusion axioms $C \subs D$, 
where $C$ and $D$ are concepts; 
an RBox $\Rcal$ which contains role inclusion axioms $S \subs R$, 
where $S$ and $R$ are roles, and role properties axioms; 
and an ABox $\Acal$ which contains assertions of the forms 
$D(a)$, $R(a,b)$, 
where $a$ and $b$ are any individual constants.}

A \emph{DL interpretation} is a pair $\I=\stru{\Delta^\I,\cdot^\I}$ where $\Delta^\I$
is a non-empty set called \emph{domain} and $\cdot^\I$ is the \emph{interpretation
function} which provides the interpretation for language elements:
$a^\I \in \Delta^\I$, for $a \in \NI$;
$A^\I \subseteq \Delta^\I$, for $A \in \NC$; 
$R^\I \subseteq \Delta^\I\times\Delta^\I$, for $R \in \NR$. 
The interpretation of complex concepts and roles is defined by the evaluation 
of their DL operators (see the paper by~\citeN{horrocks:2006} for $\SROIQ$).
%
An interpretation $\I$ \emph{satisfies} an axiom 
$\phi$, denoted
$\I\models\phi$, if it verifies the respective semantic condition, in particular: 
for $\phi = D(a)$, $a^\I \in D^\I$;
for $\phi = R(a,b)$, $\stru{a^\I,b^\I} \in R^\I$;
for $\phi = C \subs D$, $C^\I \subseteq D^\I$ (resp. for role inclusions).
$\I$ is a \emph{model} of $\K$, denoted
$\I\models\K$, if it satisfies all axioms of $\K$.
We adopt w.l.o.g.\ the {\em standard name assumption (SNA)} 
in the DL setting, i.e., every element in $\I$ is reachable via a 
distinct constant.%
\comment{We denote by $\NI_S \subseteq \NI$ the set of all such
constants, called standard names, which are uniform for all
interpretations; see the papers by \citeN{DBLP:journals/ai/EiterILST08} and \citeN{BruijnET:08}
for more details.}

Most of the following definitions for 
simple CKR are independent from the DL used
as representation language inside contexts: however,
as in the paper by~\citeN{BozzatoES:18}, we take as reference language
a restriction of the $\SROIQ$ syntax called $\SROIQrl$
which corresponds to OWL-RL. 
We restrict 
as follows \emph{left-side concepts} $C$ and \emph{right-side concepts}\/ $D$: 
\begin{center}
  $C := A \;|\; 
         \{a\} \;|\;
         C \sqcap C  \;|\; 
         C \sqcup C \;|\;
         \exists R.C \;|\; 
         \exists R.\top         
\qquad
  D := A \;|\; 
         \non C \;|\;
          D \sqcap D  \;|\; 
         \exists R.\{a\} \;|\; 
         \forall R.D \;|\, 
         \leqslant n R.\top$
\end{center}
where $A\in \NC$, $R\in \NR$ and $n \in \{0,1\}$.  
$\SROIQrl$ TBox axioms can only take the form $C \subs D$, where
$C$ is a left-side and $D$ is a right-side  
or $E \equiv F$, where $E$ and $F$ are 
both left- and right-side concepts.
A $\SROIQrl$ RBox can contain 
role inclusions $R \isa S$ (with possibly left role composition),
role disjointness, irreflexivity, symmetry, asymmetry and transitivity.
$\SROIQrl$ ABox concept assertions can only be of form $D(a)$, where
$D$ is a right-side concept.
We remark that $\SROIQrl$ basically defines a restriction
of $\SROIQ$ to axioms that are expressible as Horn rules 
(cf.\ FO translation provided by~\citeN{BozzatoES:18}). 

\smallskip\noindent
\textbf{Normal Programs and Answer Sets.}
We 
use \emph{function-free normal (datalog) rules} with \emph{(default) negation} 
under answer sets semantics~\cite{gelf-lifs-91} and 
gather them in \emph{ASP programs}.
%
A normal (datalog) rule $r$ is an expression of the form:
\begin{equation}
\label{rule}
a \leftarrow b_1, \dots, b_k, \naf b_{k+1}, \dots, \naf b_{m},\quad
0\leq k \leq m,
\end{equation}
also written $\Head(r)\leftarrow \Body(r)$ where $a, b_{1}, \dots,
b_{m}$ are function-free FO-atoms and $\naf$ is negation as failure
(NAF).
We allow that $a$ is missing ({\em constraint}),
viewing $a$ as logical constant for falsity.
A (datalog) \emph{program} $P$ is a finite set of rules.
An atom (rule etc.) is \emph{ground}, if 
no variables occur in it. 
A \emph{fact} $H$ is a ground rule $r$ with $m=0$. 
The \emph{grounding}\/ of a rule $r$, $\grd(r)$, is the set of all
ground instances of $r$, and the \emph{grounding}\/ of a program $P$
is $\grd(P) = \bigcup_{r\in P} \grd(r)$.

For any program $P$, we denote by $U_P$ its Herbrand universe and by
$B_P$ its Herbrand base; an \emph{(Herbrand) interpretation} is any 
subset  $I \subseteq B_P$ of
$B_P$.
An atom $a$ is \emph{true} in $I$, denoted $I\models a$,
if $a \in I$.
Given a rule $r \in \grd(P)$,
we say that $\Body(r)$ is true in $I$, denoted $I\models \Body(r)$, if
(i) $I\models b$ for each $b$ in $\Body(r)$ 
and (ii) $I\not\models b$ for each $\naf b$ in $\Body(r)$.
A rule $r$ is \emph{satisfied} in $I$, denoted $I\models r$, if either 
$I\models \Head(r)$ or $I\not\models \Body(r)$.
An interpretation $I$ 
is a {\em model}\/ of $P$, denoted $I \models P$,
if $I\models r$ for each $r\in \grd(P)$;
moreover, $I$ is 
\emph{minimal}, 
if $I'\not\models P$ for each subset $I'\subset I$.
Furthermore, $I$ is an \emph{answer set} of $P$, if $I$ is a minimal
model of the  (Gelfond-Lifschitz)
\emph{reduct} $G_I(P)$ of $P$ w.r.t.\ $I$, which
results from $\grd(P)$ by removing
(i) every rule $r$ such that 
$I\models l$ for some $\naf l\in\Body(r)$, and
(ii) all formulas $\naf b$ from the remaining rules.
The set of answer sets of $P$ is denoted $\mathcal{AS}(P)$.
%
%

\smallskip\noindent
\textbf{Semirings and Weighted Logic.}
A \emph{semiring} $\mathcal{R} = (R, \oplus, \otimes, e_{\oplus},
e_{\otimes})$ is a set $R\neq\emptyset$ equipped with binary operations $\oplus$ and $\otimes$, called addition and multiplication, such that (i) $(R, \oplus)$ is a commutative monoid with identity element $e_{\oplus}$, (ii) $(R, \otimes)$ is a monoid with identity element $e_{\otimes}$, (iii) multiplication left and right distributes over addition, and (iv) multiplication by $e_{\oplus}$ annihilates $R$, i.e. $\forall r \in R: r\otimes e_{\oplus} = e_{\oplus} = e_{\oplus}\otimes r$.
Examples are the natural number semiring $\mathbb{N} = (\mathbb{N}, +, \cdot, 0, 1)$ with addition and multiplication, the powerset semiring $\mathcal{P}(A) = (2^{A}, \cup, \cap, \emptyset, A)$, with union and intersection, the Boolean semiring $\mathbb{B} = (\{\mathbf{t}, \mathbf{f}\},\vee, \wedge, \mathbf{f}, \mathbf{t})$, with disjunction and conjunction, and
the tropical semiring $\mathcal{R}_{{\rm trop}} = (\mathbb{Q}\cup\{\infty\}, \min, +, \infty, 0)$, with minimum and addition.

\emph{Weighted formulas} over a semiring $\mathcal{R}$ 
and \comment{an Herbrand base} $B$ allow us to assign an interpretation $I$ a semiring value, depending on the truth of propositional variables w.r.t. $I$. Their syntax is: 
\begin{align*}
    \alpha ::= k &\mid v \mid \neg v \mid \alpha \splus \alpha \mid \alpha \stimes \alpha
\end{align*}
where $k \in R$ and $v \in B$. 
The semantics $\llbracket \alpha \rrbracket_{\mathcal{R}}(I)$ of $\alpha$ over $\mathcal{R}$ w.r.t.\ $I$ is:
{\small%
\begin{align*}
    \llbracket k \rrbracket_{\mathcal{R}} (I) &= k \text{ for } k \in R &
    \llbracket \ell \rrbracket_{\mathcal{R}} (I) &= \left\{\begin{array}{cc}
        \srone & I \models \ell\\
        \srzero & \text{ otherwise.}
    \end{array} \right. \text{for } \ell \in \{v, \neg v\}\\
    \llbracket \alpha_1 \splus \alpha_2 \rrbracket_{\mathcal{R}} (I) &= \llbracket \alpha_1\rrbracket_{\mathcal{R}} (I) \srsplus \llbracket\alpha_2 \rrbracket_{\mathcal{R}} (I) &
    \llbracket \alpha_1 \stimes \alpha_2 \rrbracket_{\mathcal{R}} (I) &= \llbracket \alpha_1\rrbracket_{\mathcal{R}} (I) \srstimes \llbracket\alpha_2 \rrbracket_{\mathcal{R}} (I)
\end{align*}}%


\section{Multi-relational simple CKR}
\label{sec:CKR}

We generalize the definition of \emph{simple CKR (sCKR)}
introduced by~\citeANP{DBLP:conf/kr/BozzatoSE18}~\shortcite{DBLP:conf/kr/BozzatoSE18,BozzatoES:CONTEXT19}  
from single- to multi-relational contextual hierarchies.
As in the original formulation of CKR by~\citeANP{BozzatoES:18}~\shortcite{BozzatoES:18,BozzatoSerafini:13}
a simple CKR is still a two layered structure, but the upper layer is simply
a poset 
with multiple orderings, corresponding to 
different contextual relations.
Simple CKRs define a core fragment of CKR allowing us
to provide lean definitions on contextual hierarchies: the presented
results, however, can be easily generalized to the full CKR.

We provide definitions for multi-relational simple CKRs
with a general set of context relations and consider the case 
for 2-relational sCKR based on \emph{temporal} and \emph{coverage} relations.

\smallskip\noindent
\textbf{Syntax.}
Consider a nonempty set $\N \subseteq \NI$  of \emph{context names}.
A \emph{contextual relation} is any strict order 
$\prec_i \subseteq \N \times \N$
over contexts. 
We may use the non-strict relation $\mlc_1\preceq_i \mlc_2$ to indicate that 
either $\mlc_1 \prec_i \mlc_2$ 
or $\mlc_1$ and $\mlc_2$ are the same context.
%
We consider two contextual relations, namely \emph{coverage} $\precc$ and \emph{temporal precedence} $\prect$. Here, $\mlc_1 \precc \mlc_2$ (resp. $\mlc_1 \prect \mlc_2$) means that $\mlc_1$ is more specific (resp. newer) than $\mlc_2$. More specific means that $\mlc_1$ represents a portion of the world covered by the 
one referred to by $\mlc_2$, as in the paper by~\citeN{serafini-homola-ckr-jws-2012}. 
%
%
%
%
We generalize the definition of defeasible axiom w.r.t. contextual relations:

\begin{definition}[r-defeasible axiom]
Given a set $\Rcal$ of contextual relations over $\N$ and a description language $\Lcals$,
an \emph{r-defeasible axiom} is any expression of the form
$\default_r(\alpha)$, where $\alpha$ is an axiom of $\Lcals$
and $\prec_r \in \Rcal$.
%
\end{definition}
Thus, we identify coverage-defeasible axioms as $\default_c(\alpha)$ and
temporal-defeasible axioms as $\default_t(\alpha)$.
%
We allow for the use of r-defeasible axioms in the local language
of contexts:

\begin{definition}[contextual language]
Given a set of context names $\N$,
for every description language $\Lcals$
we define $\Lcalsn$ as the extension of
$\Lcals$ where:
(i) $\Lcalsn$ contains the set of r-defeasible axioms in $\Lcals$;
(ii) $\eval(X,\mlc)$ is a concept (resp.\ role) of $\Lcalsn$ if
$X$ is a concept (resp.\ role) of $\Lcals$ and $\mlc \in \N$. 
\end{definition}
%
%
Using these definitions, multi-relational simple CKRs 
are defined as follows:	

\begin{definition}[multi-relational simple CKR]
A \emph{multi-relational simple CKR (sCKR)} over
$\Sigma$ and $\N$ is a structure $\CKB = \stru{\Cl, \KB_{\N}}$ where:

\begin{itemize}
  \item 
     $\Cl$ is a structure $(\N, \prec_1, \dots, \prec_m)$
		 where each $\prec_i$ is a contextual relation over
                 $\N$, and
  \item
    $\KB_{\N} = \{\KB_\mlc\}_{\mlc \in \N}$ for each context name $\sml{c} \in \N$,
    $\KB_\mlc$ is a DL knowledge base over $\Lcalsn$. 
  \end{itemize}  
\end{definition}
%
A sCKR that combines temporal and coverage orderings can be defined by 
$\Cl = (\N, \prec_t, \prec_c)$.
For simplicity, we assume that the priority for the combination of orderings is defined by the 
linear order in which they appear in 
$\Cl$: 
in the case above, we prioritize 
$\prec_t$ over 
$\prec_c$. 

\begin{example}
\label{ex:multi-relational}
  We consider the following example to explain the expected behavior
	of defeasible axioms in the case of the combination of 
	coverage and temporal relations.	
	Let us consider 
	$\CKB_{org} = \stru{\Cl, \KB_{\N}}$
	with $\Cl = (\N, \prec_t, \prec_c)$ describing the organization of a corporation.
 	The corporation has different policies with respect to 
	its local branches, represented by  
	coverage, and updates them along the time precedence. The structure of $\Cl$, 
	together with the axioms at each context,
	is shown in Figure~\ref{fig:mr-ckr-example}.
	We have a chain of three contexts (representing world, branch and
	local rules) in the direction of the coverage
	and three ``time-slices'' (2019, 2020 and 2021) along the time relation:
	\comment{thus, for example, we have 
	$\mlc_\mi{local\_2021} \prec_c \mlc_\mi{branch\_2021}$
	and
	$\mlc_\mi{branch\_2020} \prect \mlc_\mi{branch\_2019}$.}
	The corporation is active in the fields of
	Electronics ($\mi{E}$) and Robotics ($\mi{R}$) and employs supervisors ($\mi{S}$).
  \begin{figure}[t]%
  \centering
    \includegraphics[trim=182 165 195 143, clip, width=\textwidth]{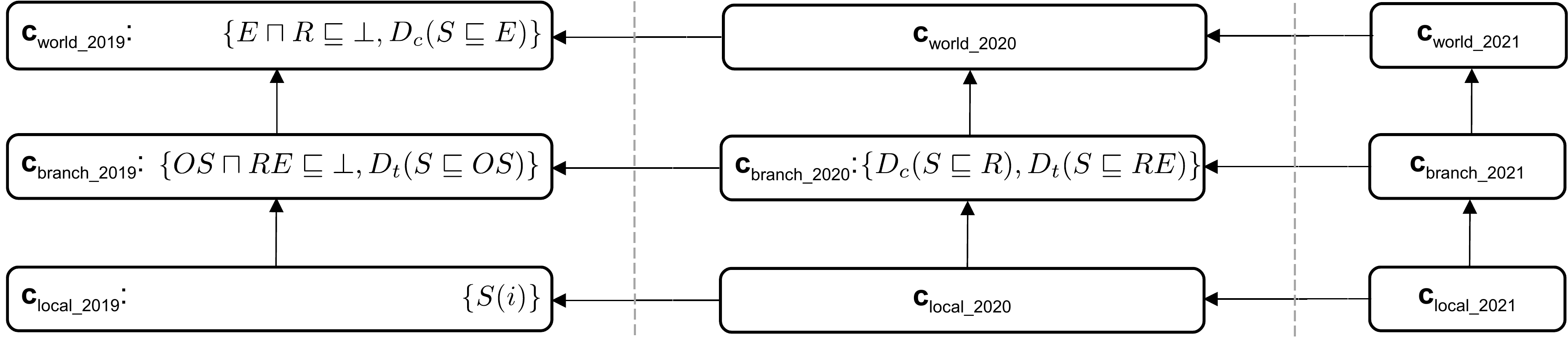}%
    \caption{Context hierarchy of multi-relational example sCKR, with axioms per context.}%
    \label{fig:mr-ckr-example}%
  \end{figure}	
  %
%
%
	%
	In $\mlc_\mi{world\_2019}$, we state that, with respect to coverage,
	every Supervisor has to be applied by default to Electronics
	and that Electronics and Robotics are disjoint.
	In the lower context $\mlc_\mi{branch\_2019}$,
	we further specify that, with respect to time,
	Supervisors have to work by default \emph{OnSite} ($OS$)
	(where working \emph{OnSite} and \emph{Remote} ($RE$) are disjoint).
	In 2019's local context
	$\mlc_\mi{local\_2019}$ we assert that $i$
	is a Supervisor.
	The previous defeasible statements are, however, 
	contradicted by the ones in $\mlc_\mi{branch\_2020}$,
	where 
  Supervisors are applied to Robotics
	and work on Remote.

	The interpretation of defeasible propagation and preferences,
	then, should define the interpretation of what is derivable in the 
	local context in the three time-slices.
	In $\mlc_\mi{local\_2019}$ no overriding takes place; then
	we should derive $E(i), OS(i)$.
	In $\mlc_\mi{local\_2020}$ the 
	more coverage-specific axiom in $\mlc_\mi{branch\_2020}$
	is preferred,
	thus we derive $R(i)$; 
	the time-related defeasible axiom $\default_t(\mi{S} \subs \mi{RE})$
	is applied locally 
	to the 2020 time-slice,
	thus we derive $RE(i)$.
	In the 2021 time-slice no new information is provided, thus
	the overriding preferences should enforce that the more specific and 
	recent information is used:
	in $\mlc_\mi{local\_2021}$ we expect to derive
	$R(i), RE(i)$. \EndEx
\end{example}



\smallskip\noindent
\textbf{Semantics.}
%
A sCKR interpretation gathers interpretations for the local contexts
as follows.
\begin{definition}[sCKR interpretation]
\label{def:ckr-int}
  An interpretation for $\Lcalsn$
  is a family $\IC = \{\Ic\}_{\mlc\in\N}$ of $\Lcals$ interpretations, such that
  $\Delta^{\Ic}\,{=}\, \Delta^{\Icp}$ and $a^{\Ic} \,{=}\, a^{\Icp}$, for
    every $a \,{\in}\, \NI$ and 
    $\mlc,\mlc' \,{\in}\, \N$.
\end{definition}
The interpretation of concepts and role expressions in $\Lcalsn$
is obtained by extending the standard interpretation 
to $\eval$ expressions:
for every $\mlc \in \N$, $\eval(X,\mlc')^{\Ic} = X^{\Icp}$.
%
We consider the definition of axiom instantiation 
provided by~\citeN{BozzatoES:18}:
%
given an axiom $\alpha \in \Lcal_\Sigma$ with FO-translation
$\forall\vc{x}.\phi_\alpha(\vc{x})$, the \emph{instantiation} of $\alpha$
with a tuple $\ee$ of individuals in $\NI$, written $\alpha(\ee)$, is the
specialization of $\alpha$ to $\ee$, i.e., $\phi_\alpha(\ee)$,
depending on the type of $\alpha$.

For a structure $\Cl = (\N, \prec_1, \dots, \prec_m)$ and $1 \leq i \leq m$, we denote by $\preceq_{-i}$ the order obtained as the reflexive and transitive closure of $\bigcup_{j \neq i} \prec_{j}$, i.e., the union of all orders $\prec_{j}$ except for $\prec_{i}$. 
We denote by $\preceq_{*}$ the order obtained as the reflexive and transitive closure of the union of all 
$\prec_{j}$.

\begin{definition}[clashing assumptions and sets]
A \emph{clashing assumption} for a context $\mlc$ and contextual relation $r$
is a pair $\stru{\alpha, \ee}$
such that $\alpha(\ee)$ is an axiom instantiation of $\alpha$, and
$\default_r(\alpha) \in \KB_{\mlc'}$ is a defeasible axiom of some 
\mbox{$\mlc' \succeq_{-r} \mlc'' \succ_r \mlc$.}
%
A \emph{clashing set} for 
$\stru{\alpha,\ee}$
is a satisfiable set $S$ of ABox assertions s.t.
 $S \cup \{\alpha(\ee)\}$ is unsatisfiable.
\end{definition}
A clashing assumption $\stru{\alpha, \ee}$
represents that
$\alpha(\ee)$ is not satisfiable in context $\mlc$, 
and a clashing set $S$ provides a ``justification'' for the
local assumption of overriding of $\alpha$ on $\ee$. 
%
CAS-interpretations 
include a set of clashing assumptions for each context and contextual relation:

\begin{definition}[CAS-interpretation]
  A \emph{CAS-interpretation} is a structure 
	$\ICAS=\stru{\IC,\ov{\casmap}}$ where
  $\IC$ is an interpretation and
	$\ov{\casmap} = \{\casmap_1, \dots, \casmap_m\}$ such that each 
	$\casmap_i$, for $i \in \{1, \dots, m\}$, maps every $\mlc \in \N$
  to a set $\casmap_i(\mlc)$ of clashing assumptions for context $\mlc$
	and context relation $\prec_i$. 
\end{definition}
%
Satisfaction of a sCKR $\CKB$ needs to consider the 
effect of the different relations:
%
%
\begin{definition}[CAS-model]
\label{def:cas-model}
Given a multi-relation sCKR $\CKB$,
a CAS-interpretation 
$\ICAS = \stru{\IC, \ov{\casmap}}$
is a \emph{CAS-model} for $\CKB$ (denoted $\ICAS \models \CKB$), 
if the following holds\footnote{Here, it is important to ensure that (defeasible) axioms are correctly propagated w.r.t.\ \emph{any} context relation $\prec_{i}$.}: 
\begin{enumerate}[label=(\roman*)]
 \item
   for every $\alpha \in \KB_\mlc$ (strict axiom), 
   and $\mlc' \preceq_{*}\mlc$, $\Icp \models \alpha$;
 \item
   for every $\default_i(\alpha) \in \KB_\mlc$ 
	 and $\mlc' \preceq_{-i} \mlc$, $\I(\mlc') \models \alpha$;
  \item
    for every $\default_i(\alpha) \in \KB_\mlc$ 
    and $\mlc'' \prec_i \mlc' \preceq_{-i} \mlc$,    
		if $\stru{\alpha,\vc{d}} \notin \casmap_i (\mlc'')$,
    then $\I(\mlc'') \models \phi_\alpha(\vc{d})$. 
\end{enumerate}
\end{definition}
Intuitively: (i) strict axioms are propagated
across the hierarchy structures over $\preceq_{*}$ from higher to lower contexts;
(ii) considering contexts that are related by relations other than $\prec_i$
(including the context in which axioms are declared), 
defeasible axioms
$\default_i(\alpha)$ are interpreted as strict axioms; 
(iii) over relation $\prec_i$,
axioms 
$\default_i(\alpha)$
are verified in context $\mlc''$ only if
applied to instances $\vc{d}$
that are not in the clashing assumptions for $\mlc''$ and relation $\prec_i$.
Note that 
these propagation rules are applied for 
every contextual relation: however, the definition can be easily
extended to assign different conditions for propagation and overriding
for each of the orderings.

We provide a \emph{local preference}\/ on clashing assumption sets for each of the relations:

\begin{list}{\emph{(LP).}}{\setlength{\topsep}{2pt}
\setlength{\leftmargin}{14pt}
\setlength{\itemsep}{0pt}
\setlength{\itemindent}{10pt}}
\item[\emph{(LP).}]  $\chi_i^1(\mlc) > \chi_i^2(\mlc)$, if for every $\stru{\alpha_1,\ee} \in \chi_i^1(\mlc) \setminus \chi_i^2(\mlc)$ with 
  $\default_i(\alpha_1)$ at a context $\mlc_1 \succeq_{-i} \mlc_{1b} \succ_{i}  \mlc$, 
  some $\stru{\alpha_2,\ff} \in \chi_i^2(\mlc) \setminus
  \chi_i^1(\mlc)$ exists with 
	$\default_i(\alpha_2)$ at context $\mlc_2 \succeq_{-i} \mlc_{2b} \succ_i \mlc$ 
	s.t.\ $\mlc_{1b} \succ_i \mlc_{2b}$.
\end{list}
Intuitively, $\chi^1_i(\mlc)$ is preferred to $\chi^2_i(\mlc)$ if 
$\chi^1_i(\mlc)$ exchanges the ``more costly'' exceptions of $\chi^2_i(\mlc)$
at more specialized contexts with ``cheaper'' ones at more general contexts. 
%
As above, 
multiple options for 
local preference can be adopted, 
cf.~\cite{DBLP:conf/kr/BozzatoSE18} 
for ranked hierachies.


Two DL interpretations $\I_1$ and $\I_2$ are
\emph{$\NI$-congruent}, if
$c^{\I_1} = c^{\I_2}$ holds for every $c\in \NI$.
This extends to CAS interpretations $\IC_{\CAS} = \stru{\IC, \ov{\casmap}}$ by considering all
context interpretations $\I(\mlc) \in \IC$.

\begin{definition}[justification]
	We say that $\stru{\alpha, \ee} \in \chi_i(\mlc)$ is
	\emph{justified} for a $\CAS$ model $\IC_{\CAS}$,   
	if some clashing set
	$S_{\stru{\alpha,\ee},\mlc}$ exists
	such that, for every 
	$\IC_{\CAS}' = \stru{\IC',\ov{\casmap}}$ of $\CKB$ that is $\NI$-congruent with $\IC_{\CAS}$, 
	it holds that $\I'(\mlc) \models S_{\stru{\alpha,\ee},\mlc}$.
	A $\CAS$ model $\ICAS$ of 
  a sCKR $\CKB$ is \emph{justified}, 
  if every $\stru{\alpha, \ee} \in \ov{\casmap}$ is justified in $\CKB$.
\end{definition}
%
%
We define a \emph{model preference} 
by combining the preferences of the relations:
it is a global lexicographical ordering on models where
each $\prec_i$ defines the ordering at the $i$-th position.

%
\smallskip\noindent
\emph{(MP).}\
	$\IC^1_{\CAS} = \stru{\IC^1, \casmap^1_1, \dots, \casmap^1_m}$ 
	is preferred to $\IC^2_{\CAS} = \stru{\IC^2, \casmap^2_1, \dots, \casmap^2_m}$ if
  \begin{enumerate}[label=(\roman*)]
	\item 
	   there exists $i \in \{1, \dots, m\}$ and some $\mlc \in \N$ s.t. 
	   $\casmap^1_i(\mlc) > \casmap^2_i(\mlc)$ and not $\casmap^2_i(\mlc) > \casmap^1_i(\mlc)$, and
     for no context $\mlc' \neq \mlc \in \N$ it holds that 
	   $\casmap^1_i(\mlc') < \casmap^2_i(\mlc')$ and not $\casmap^2_i(\mlc') < \casmap^1_i(\mlc')$.
  \item
	   for every $j < i \in \{1, \dots, m\}$, 
		 it holds $\casmap^1_j \approx \casmap^2_j$
		 (i.e.\ 
		 (i) or its converse 
		 do not hold
		 for $\prec_j$).
	\end{enumerate}


\noindent
Then, CKR models are defined by taking into
account justification and model preference.

\begin{definition}[CKR model]
\label{def:ckr-model}
  An interpretation $\IC$ is a \emph{CKR model} of a sCKR $\CKB$ 
	(in symbols, $\IC\models\CKB$) if:
 (i) $\CKB$ has some justified CAS model $\ICAS = \stru{\IC, \ov{\casmap}}$;	  
(ii) there exists no justified 
     $\ICAS' = \stru{\IC', \ov{\casmap}'}$
     that is preferred to $\ICAS$.
\end{definition}

\begin{example}
   By considering the sCKR of Example~\ref{ex:multi-relational},
   we can show how the preference
	 for different relations influences the global 
	 model preference.
	 In the case of $\CKB_{org}$, 
	 we have 8 justified interpretations that are based
	 on combinations of the following clashing assumption sets (for both relations)
	 on contexts $\mlc_\mi{local\_2020}$ and $\mlc_\mi{local\_2021}$.
	 \iftrue
	 For any CAS model $\chi_{t}(\mlc_\mi{local\_2020}) = \chi_{t}^{0}(\mlc_\mi{local\_2020}) = \{\stru{S \subs OS, i}\}$
	 and $\chi_{c}(\mlc_\mi{local\_2020})$ is either $\chi_{c}^{0}(\mlc_\mi{local\_2020}) = \{\stru{S \subs E, i}\}$ or
	 $\chi_{c}^{1}(\mlc_\mi{local\_2020}) = \{\stru{S \subs R, i}\}$.
	 
	 For $\mlc_\mi{local\_2021}$ we have that $\chi_{t}(\mlc_\mi{local\_2021})$ is either 
	 $\chi_{t}^{1}(\mlc_\mi{\local\_2021}) = \{\stru{S \subs OS, i}\}$ or\linebreak
	 $\chi_{t}^{2}(\mlc_\mi{\local\_2021}) = \{\stru{S \subs RE, i}\}$. For $\chi_{c}(\mlc_\mi{local\_2021})$ we have the same choices as for $\chi_{c}(\mlc_\mi{local\_2020})$.
	 
	 According to the (LP) 
	 definition,
	 $\chi_{c}^{0}(\mlc_\mi{local\_2020}) > \chi_{c}^{1}(\mlc_\mi{local\_2020})$
	 since $D_{c}(S \subs E)$ occurs at a less specific context 
	 w.r.t.\ $\precc$ than $D_{c}(S \subs R)$. 
	 Similarly, \comment{$\chi_{t}^{1}(\mlc_\mi{\local\_2021}) > \chi_{t}^{2}(\mlc_\mi{\local\_2021})$}.
	 
	 Since we can choose the clashing assumptions per context independently,
	 the clashing assumption map of CKR models is uniquely determined by (MP) as
	 $\ov{\chi} = \stru{\chi_t,\chi_c}$
	 where $\chi_t = \chi_{t}^{0} \cup \chi_{t}^{1}$ and 
	 $\chi_c = \chi_{c}^{0} \cup \chi_{c}^{2}$.
	 \fi
	 Indeed, this corresponds to the intuitive
	 model where overridings over temporal relation
	 occur on defeasible axioms in the ``older'' contexts
	 and in the ``higher'' contexts with respect to the
	 coverage relation.\EndEx
\end{example}

%

\smallskip\noindent
\textbf{Reasoning and Complexity.}
%
We consider the following \emph{reasoning tasks} for sCKR:
%
\begin{itemize}
\item
  \hspace*{-1pt}\mbox{{\em $\smlc$-entailment} $\CKB\models \smlc\,{:}\,\alpha$, denoting that axiom $\alpha$ 
        is entailed in each CKR-model of $\CKB$ at context
	$\smlc$.}
\item 
\hspace*{-1pt}%
	{\em Boolean conjunctive query (BCQ) answering} 
	$\CKB\models \exists\vc{y}\gamma(\vc{y})$, where $\gamma(\vc{y}) = \gamma_1\land\cdots\land\gamma_m$ is 
	an \mbox{existentially closed conjunction 
        of atoms $\gamma_i\,{=}\, \smlc_i{:}\alpha_i(\vc{t}_i)$
	with context name $\smlc_i$ and assertion $\alpha_i(\vc{t}_i)$.} 
\end{itemize}

\noindent
%
The complexity of reasoning with contextual hierarchies in sCKR was
studied by~\citeANP{DBLP:conf/kr/BozzatoSE18} \shortcite{DBLP:conf/kr/BozzatoSE18,BozzatoES:CONTEXT19}
in particular, CKR satisfiability is NP-complete 
while CKR model checking is coNP-complete already for ranked
hierarchies. 
This causes the complexity of $\smlc$-entailment
to increase 
in presence of hierarchies: 
for polynomial-time 
local preferences 
on overridings, $\smlc$-entailment is $\Pi^p_2$-complete. 
In contrast, BCQ answering remains $\Pi^p_2$-complete
as verifying a guess for a countermodel to the query remains in coNP.
%
These results 
would carry over 
to multi-relational hierarchies:
for combinations of polynomial-time 
preferences 
(like the global preference we considered),
$\smlc$-entailment and similarly BCQ answering would still be $\Pi^p_2$-complete. 

\section{Preferences with Algebraic Measures}
\label{sec:al-measures}

The question rises how the reasoning problems 
above can be expressed and solved. Previously, in the case of sCKRs
with a single relation the strategy was to encode the problem in ASP
using a program whose stable models correspond to the least justified 
models of the sCKR. The preferred models, i.e.\ sCKR models, were then
selected using weight constraints in the restricted case of 
ranked hierarchies \cite{BozzatoES:18} or by using a dedicated
algorithm for general hierarchies \cite{BozzatoES:CONTEXT19}. The
preference over models for multi-relational sCKRs is more complicated
and thus not easily expressed with weight constraints: 
we can leverage the power of quantitative extensions of ASP to express
model preferences induced by multi-relational sCKRs.

The recently introduced algebraic measures for ASP, which connect ASP
with weighted formulas, were shown to be a general framework for specifying
quantitative reasoning problems \cite{EiterK:20}. Also preferential
reasoning falls into this category, thus allowing us to
use algebraic measures to specify a preference on the answer sets in
such a way that the preferred answer sets correspond to the preferred
least justified models. The concept is as follows.


\begin{definition}[Algebraic Measure]
An \emph{algebraic measure} $\mu = \langle \Pi, \alpha, \mathcal{R}\rangle$ consists of an answer set program $\Pi$, a weighted formula $\alpha$, and a semiring $\mathcal{R}$. The weight of an answer set $S \in \mathcal{AS}(\Pi)$ is
$
\mu(S) =  \llbracket \alpha \rrbracket_{\mathcal{R}}(S).
$
And the \emph{overall weight} of $\mu$ is defined as 
$
\mu(\Pi) = \srbplus_{S \in \mathcal{AS}(\Pi)} \mu(S).
$
\end{definition}
Intuitively, given $\mu = \langle \Pi, \alpha, \mathcal{R}\rangle$ the
program $\Pi$ specifies which interpretations are accepted and the
weighted formula $\alpha$ measures some value associated with
them. Using algebraic measures, we can not only assign answer sets a
weight but also obtain some information from all answer sets by
considering the overall weight $\mu(\Pi)$.
\begin{example}
\label{ex:measure}
Let $\Pi$ be some answer set program. Then, for example,
	for $\mu_1 = \langle \Pi, 1, \mathbb{N}\rangle$ the overall weight $\mu(\Pi)$ is the number of answer sets of $\Pi$.
	For $\mu_2 = \langle \Pi, (a_1 \stimes 1 \splus \neg a_1)
        \stimes \dots \stimes (a_n \stimes 1 \splus \neg a_n),
        \mathcal{R}_{\max}\rangle$\comment{, where $\mathcal{R}_{\max} = \langle \mathbb{R}\cup\{-\infty\}, \max, +, -\infty, 0\rangle$,} the weight $\mu(S)$ of an answer
        set $S$ is the number of atoms $a_1, \dots, a_n$ it satisfies. \comment{We need the additional term $\neg a_i$, since $a_i \stimes 1$ evaluates to $\srzero = -\infty$ when $a_i$ is false and not to the desired value $\srone = 0$. Due to the usage of the semiring $\mathcal{R}_{\max}$ the overall weight $\mu(\Pi)$ is the maximum number of atoms from $a_1, \dots, a_n$ that are satisfied in any answer set of $\Pi$.}
\end{example}
A natural use case of algebraic measures is preferential reasoning.
In the sequel, a \emph{preference relation} is any
asymmetric relation. 

\begin{definition}[Preferred Answer Set]
Given a measure $\mu = \langle \Pi, \alpha, \mathcal{R}\rangle$ and
a preference relation $>$ on $R$, an answer set $S \in \mathcal{AS}(\Pi)$ is \emph{preferred}\/
w.r.t.\ $\mu$ and $>$\, if no $S' \in \mathcal{AS}(\Pi)$ exists such
that $\mu(S') > \mu(S)$.
\end{definition}
Intuitively, we use $\mu$ as an optimization function and take the
preferred answer sets as those that achieve an optimal value.
\begin{example}
Reconsider the measure $\mu_2$ from Example~\ref{ex:measure}. If
$a_1, \dots, a_n$ are desired to be true, then we only want to
consider those answer sets for which a maximal number of them is
true. These are exactly the preferred answer sets with respect to the
measure $\mu_2$ and the usual order over the reals.
\end{example}
We assume a program $PK(\CKB)$ (see Section~\ref{sec:asprin_enc}), 
which intuitively guesses a set of atoms $\ovr(\phi, \ee, \mlc, i)$, 
each corresponding to a clashing assumption $\stru{\phi, \ee}$ in
$\casmap_i(\mlc)$, and checks whether there is an CAS model 
$\IC_{CAS} = \stru{\IC, \overline{\casmap}}$. The
answer sets $I$ corresponds to the least CAS models with that property.
Then we can
introduce a measure $\mu_{opt}$ and order $>_{opt}$ to obtain those
answer sets of $PK(\CKB)$ as preferred answer sets w.r.t.\ $\mu_{opt}$
and $>_{opt}$ that correspond to the preferred least justified models
of $\CKB$. Here, we do not require any restrictions on the $\CKB$ at
all.

We use the powerset semiring $\mathcal{P}(CA)$ over the set $CA$, which contains the tuple $\stru{\phi, \ee, \mlc, i}$ for each possible clashing assumption $\stru{\phi, \ee}$ that can occur at context $\mlc$ w.r.t.\ relation $i$. The weighted formula of $\mu_{opt} = \stru{PK(\CKB), \alpha, \mathcal{P}(CA)}$ is given by
$
\alpha = \bplus_{\stru{\phi, \ee, \mlc, i} \in CA} \ovr(\phi, \ee, \mlc, i) \stimes \{\stru{\phi, \ee, \mlc, i}\}.
$
It is easy to see that for each answer set $I$ of $PK(\CKB)$ it holds that $\stru{\phi, \ee, \mlc, i} \in \mu_{opt}(I)$ iff $\ovr(\phi, \ee, \mlc, i)  \in I$. Thus, we only need to define the order $>_{opt}$ on the semiring values $S \subseteq CA$ that correctly captures the ordering on the justified models. For this we let $S \subseteq CA$ and define $(\chi^{(S)}_{i})_{i \in [m]}$, the clashing assumption maps corresponding to $S$, by setting 

\smallskip

\centerline{$
\casmap^{(S)}_{i}(\mlc) = \{\stru{\phi, \ee} \mid \stru{\phi, \ee, \mlc, i} \in S\}.
$}

\smallskip
\noindent Then for $S, S' \subseteq CA$, we define $S >_{opt} S'$ iff
\begin{enumerate}[label=(\roman*)]
	\item 
	   there exists $i \in \{1, \dots, m\}$ and some $\mlc \in \N$ s.t. 
	   $\casmap^{(S)}_i(\mlc) > \casmap^{(S')}_i(\mlc)$ and not $\casmap^{(S')}_i(\mlc) > \casmap^{(S)}_i(\mlc)$, and
     for no context $\mlc' \neq \mlc \in \N$ it holds that 
	   $\casmap^{(S)}_i(\mlc') < \casmap^{(S')}_i(\mlc')$ and not $\casmap^{(S')}_i(\mlc') < \casmap^{(S)}_i(\mlc')$.
  \item
	   for every $1\leq j < i \leq m$, 
	   we have $\casmap^{(S)}_j\,{\approx}\, \casmap^{(S')}_j$
		 (i.e.\ 
                 (i) or its converse is unprovable for $\prec_j$).
\end{enumerate}
\begin{theorem}
Let $\CKB$ be an sCKR and $PK(\CKB)$ as described above. 
Then the preferred answer sets w.r.t.\ $\mu_{opt}$ and $>_{opt}$ correspond
to the least CKR models $\stru{\IC, \overline{\casmap}}$ of $\CKB$, i.e. those where $\IC$ is the $\subseteq$-minimal interpretation such that $\stru{\IC, \overline{\casmap}}$ is a CKR model.
\end{theorem}
In the following we outline how such an ASP program $PK(\CKB)$ can be constructed.
Furthermore, we show that for suitably restricted $\CKB$, 
we can also express 
algebraic
measures and preferential answer sets using
asprin. 

\section{ASP Encoding of Reasoning Problems}

%
%
%
\noindent\textbf{ASP translation process.}
The ASP translation by \citeN{BozzatoES:18} 
for instance checking (w.r.t.\ $\mlc$-entailment, under UNA) in a 
$\SROIQrl$ CKR can be extended to 
multi-relational sCKRs
$\CKB = \stru{\Cl, \KB_{\N}}$, such that 
(1) a set of \emph{input rules $I$} encode the contextual structure
and local contents of contexts in $\CKB$ as facts; (2) uniform \emph{deduction rules $P$}
encode the interpretation of 
axioms; and (3)
the instance query
is encoded by \emph{output rules $O$} as ground facts.
%

Formally, the \emph{CKR program} $PK(\CKB)  = PG(\Cl) \cup
\bigcup_{\mlc \in \N} PC(\mlc, \CKB)$ encodes the whole sCKR, 
where  	  $PG(\Cl) = I_{glob}(\Cl) \cup P_{glob}$
is the \emph{global program} for $\Cl$ and 
$PC(\mlc, \CKB) = I_{loc}(\KB_\mlc, \mlc) \cup P_{loc}$ is the 
 \emph{local program} for $\mlc \in \N$.
Query answering $\CKB \models \mlc: \alpha$ is then achieved
by testing whether the instance query, translated to $O(\alpha, \mlc)$, is a consequence of 
the \emph{preferred} models of $PK(\CKB)$, 
i.e., whether $PK(\CKB) \cup P_{pref} \models O(\alpha, \mlc)$ holds,  
where $P_{pref}$ are the newly added rules for selection of preferred models.
Analogously, this can be extended to conjunctive queries as shown by~\citeN{BozzatoES:18}.
The details of the translation rules are in the Appendix; 
in the following, we further discuss $P_{pref}$.

\smallskip\noindent
\textbf{asprin-based model selection.}
\label{sec:asprin_enc}
From the translation $PK(\CKB)$ we obtain the least justified models of $\CKB$ 
as answer sets of an ASP program. In Section~\ref{sec:al-measures}, we showed
how to use algebraic measures for describing which answer sets correspond to preferred models.
By suitably 
restricting the input CKR $\CKB$,
we show that we can implement the preference already 
in the \emph{asprin}\/ framework~\cite{DBLP:conf/aaai/BrewkaD0S15}.
The latter can not express sCKR preference relations in general as 
\eval-expressions 
may cause non-transitive and even cyclic preference relations. 
We thus restrict the use of \eval-expressions such that we can define
an asprin preference relation $>$ that has the same preferred answer
sets as $\mu_{opt}$ but is a strict partial order.
%
For this, we consider a dependency graph.
\begin{definition}[Dependency Graph]
The \emph{dependency graph}\/ of an sCKR $\CKB$ is the directed graph $DEP(\CKB) = (V,E)$ is $\CKB$, where:%
\begin{itemize}
	\item $V = \{X_{\mlc} \mid X \text{ is a concept or role that
          occurs in } \KB_\mlc\}$, i.e., we have a vertex
          $X_{\mlc}$ for every combination of a concept or role $X$ that occurs in $\CKB$ and context $\mlc \in \N$.
	\item $(X_{\mlc}, X'_{\mlc'}) \in E$ if either: 
		  (i) $\mlc = \mlc'$, $X$ is a complex concept or role and $X'$ is a subexpression of $X$;
		  (ii) $\mlc = \mlc'$ and $X,X'$ co-occur in some
        (possibly defeasible) axiom; or
		  (iii) $X = \eval(X', \mlc')$.
\end{itemize} 
\end{definition}
Intuitively, a path connects two concepts/roles $X_{\mlc}, X'_{\mlc'}$
in $DEP(\CKB)$ if the interpretations of $X, X'$ at contexts $\mlc,
\mlc'$,  respectively, 
may depend on each other. 
If there are no \eval-expressions, then clearly there is no path between $X_{\mlc}, X'_{\mlc'}$
when $\mlc \neq \mlc'$. 
In this case,  we can choose the interpretations
per context
independently,
which simplifies the choosing of preferred interpretations significantly.
However,
as the preference only refers to clashing assumptions caused by
defaults, we can also use 
a weaker condition to a similar effect:
\begin{definition}[\eval-Disconnectedness]
Let $\CKB$ be an sCKR and $X, X'$ two concepts or roles that occur in
        default axioms. Then $X, X'$ are \emph{\eval-disconnected} if
        there is no path between $X_{\mlc}, X'_{\mlc'}$ in $DEP(\CKB)$
        for every $\mlc \neq \mlc'$. Furthermore, $\CKB$ is
        \eval-disconnected if
every such $X, X'$ are \eval-disconnected.
\end{definition}

\noindent
In the following, we 
confine to \eval-disconnected sCKR's
and define the preference in asprin as follows. 
We use so called ``poset'' preferences, which are specified using statements of the form:
\begin{lstlisting}[mathescape=true]
#preference($p$,poset){  $F_1$ >> $F_2$;   $F_3$ >> $F_4$;   $\dots$  ;   $F_{2n-1}$ >> $F_{2n}$ }.
\end{lstlisting}
Here each $F_i$ is a Boolean formula,  
and a partial order $>$ on such formulas is defined
by the transitive closure of \lstinline*>>*. An interpretation $X$ is preferred over interpretation $Y$ w.r.t.\ $p$ (written $X >_p Y$) if 
(i) for some $i$, $X \models F_i$ and $Y \not\models F_i$, and 
(ii) for every $i$ s.t.\ $Y \models F_i$ and $X \not\models F_i$,
some $j$ exists s.t.\ $F_j > F_i$ and $X \models F_j$ and $Y \not\models F_j$.
%

We then define the local preference w.r.t.\ context $\mlc$ and
relation $i$ by
\begin{lstlisting}[mathescape=true]
#preference(LocPref($\mlc$,$i$),poset){ 
    $\neg \ovr(\alpha, X, \mlc, i)$  >> $\ovr(\alpha, X, \mlc, i)$;
    $\neg \ovr(\alpha_2, Y, \mlc, i)$ >> $\neg \ovr(\alpha_1, X, \mlc, i)$;  for $\mlc_1 \succeq_{-i} \mlc_{1b} \succ_{i}  \mlc$ and $\mlc_2 \succeq_{-i} \mlc_{2b} \succ_i \mlc$  and $\mlc_{1b} \succ_i \mlc_{2b}$ and $D_{i}(\alpha_i)$ in $\KB_{\mlc_i}$. $\hspace{5.2cm}$}.
\end{lstlisting}
This encodes that, whenever possible, we prefer not to override a
defeasible axiom $D_{i}(\alpha)$ (line 2); further, if we have to
override some defeasible axiom, then we prefer to override the least
specific one possible (line 3). 
%
%
Next, 
we emulate the preference definition (MP), where item (i)
combines the local preferences into a preference per defeasibility relation and item (ii) states that the global ordering is the lexicographical combination of the preferences per relation. 

Using asprin, we can combine existing preference orders into a new one. 
This is where \eval-disconnectedness comes into play. 
While for general sCKRs this is not the case, for \eval-discon\-nected sCKRs, 
the preferred models w.r.t.\ (i) are the \emph{pareto optimal} models
$X$, i.e., 
no model $Y$ exists that is strictly better than $X$ on one of the local preferences 
\lstinline*LocPref($\mlc,i$)* and at least as good on all the others.
Thus, we use the \lstinline*pareto* 
type to define the preference per relation $i$:
\begin{lstlisting}
#preference(RelPref($i$),pareto){**LocPref(C,$i$) : context(C)}.
\end{lstlisting} 
Here, the condition \lstinline*context(C)* enforces that we take the pareto order over the orders \lstinline*LocPref(C,$i$)* for every context \lstinline*C*. 
Finally, for (ii), we use asprin's \emph{lexicographical preference} over orders $(p_i)_{i \in [n]}$
with weights $(w_i)_{i \in [n]}$. 
When $w_i > w_j$ we may worsen $p_j$ to improve $p_i$.
\begin{lstlisting}[mathescape=true]
#preference(GlobPref,lexico){W::**RelPref(I) : rel_w(I,W)}.
\end{lstlisting} 
Similar to above, the condition \lstinline*rel_w(I,W)* ensures that we
obtain the lexicographical order over all preferences
\lstinline*RelPref(I)*, where \lstinline*I* is a relation with weight
\lstinline*W*; in our case,
\lstinline*W* is its index. 

\smallskip\noindent
\textbf{Correctness.}
%
The presented encoding yields
a sound and complete reasoning method
for multi-relational sCKRs in $\SROIQrld$ \emph{normal form}, on time and
coverage relations. $\SROIQrld$ disallows defeasible $\SROIQrl$-axioms 
that introduce disjunctive information. The normal form of $\SROIQrld$ due to~\citeN{BozzatoES:18}
is summarized in the Appendix. Formally, 
\begin{theorem}
\label{thm:encode}
Let  $\CKB$ be a multi-relational sCKR that is \eval-disconnected  and in $\SROIQrld$ normal form. 
Then under the unique name assumption (UNA),
\begin{enumerate}[label=(\roman*)]
\item 
for every $\alpha$ and $\mlc$ such that $O(\alpha, \mlc)$ is defined, 
$\CKB \models \mlc:\alpha$ iff $PK(\CKB) \cup P_{pref} \models O(\alpha, \mlc)$;
\item 
for every BCQ $Q = \exists\vc{y}\gamma(\vc{y})$ on $\CKB$,
$\CKB \models Q$ iff $PK(\CKB) \cup P_{pref} \models O(Q)$.
\end{enumerate}
\end{theorem}
Similarly to~\citeANP{DBLP:conf/kr/BozzatoSE18}~\shortcite{BozzatoES:CONTEXT19,DBLP:conf/kr/BozzatoSE18},
the result is shown by proving 
a correspondence between the least CAS models of $\CKB$
and the answer sets of $PK(\CKB)$, 
and then between preferred CAS models and answer sets, which are here selected by our asprin preference.
For space reasons, we confine to a proof outline; more details 
are given in the Appendix.

Without loss of generality, we can restrict to \emph{named models}, i.e., models $\Ical$ s.t.\ the interpretation of atomic concepts and
roles belongs to $N^\Ical$ for some $N \subseteq \NI \setminus \NI_S$.
This allows us to concentrate on Herbrand models for $\CKB$; in particular,
w.r.t.\ a clashing assumption $\ov{\casmap} = (\casmap_t, \casmap_c)$, we have a least Herbrand model which we denote as $\hat{\IC}(\ov{\casmap})$.


Suppose $\IC_{\CAS} = \stru{\IC, \ov{\casmap}}$ 
is a justified named CAS-model.
We can build from $\IC_{\CAS}$ a corresponding Herbrand interpretation 
$I(\IC_{\CAS})$ for the program $PK(\CKB)$.
%
%
%
Along the lines of~\citeN[Lemma 6]{BozzatoES:18}, 
we can then show that the answer sets of $PK(\CKB)$ 
coincide with the sets $I(\hat{\IC}(\ov{\casmap}))$ where $\ov{\casmap}$ 
is the clashing assumption of a named CAS model of $\CKB$. With this in place, we show that 
%
in case of a multi-relational hierarchy, 
the answer sets of $PK(\CKB) \cup P_{pref}$ found optimal by the asprin 
preference $\lf{GlobPref}$ (implementing $P_{pref}$)
coincide with the sets $I(\hat{\IC}(\ov{\casmap}))$ where $\ov{\casmap}$ 
is the clashing assumption of a named preferred CAS model (i.e.\ CKR model) of $\CKB$. 
%
%
%

\smallskip\noindent
\textbf{Prototype Implementation.}
%
The ASP translation presented above
is implemented as a proof-of-concept in 
the \CKRew \emph{(CKR datalog rewriter)} 
prototype~\cite{BozzatoES:18}.
%
\CKRew is a Java-based command line application that builds on dlv.
It accepts as input RDF files representing
the contextual structure and local knowledge bases
and produces as output a single {\tt .dlv} text file with
the ASP rewriting for the input CKR.
The latest version of \CKRew 
is available 
at \url{github.com/dkmfbk/ckrew/releases}
and includes sample RDF files for 
$\CKB_{org}$ of Example~\ref{ex:multi-relational}.


\section{Additional Possibilities with Algebraic Measures}

We 
highlight further fruitful usages
of algebraic measures 
for reasoning with sCKRs.

\smallskip\noindent
\textbf{Preferred Model as an Overall Weight.}
%
First, we show another alternative way of obtaining a preferred model as the result of an overall weight query. Formally, we have the following:

\begin{theorem}
Let $\CKB$ be a single-relational, \eval-free sCKR. Then there exist a
semiring $\mathcal{R}_{\rm one}(\CKB)$ and weighted formula
$\alpha_{\rm one}$ such that the overall weight of $\mu_{\rm one} =
\langle PK(\CKB), \alpha_{\rm one}, \mathcal{R}_{\rm
  one}(\CKB)\rangle$ is either $(I, \casmap)$, where $I$ is the minimum
lexicographical preferred answer set of $PK(\CKB)$ and $\casmap$ is
the corresponding clashing assumption map, or $\zero$ if there is no
preferred answer set.
\end{theorem}
Here, the lexicographical order $>_{lex}$ over answer sets is given by
$I >_{lex} I'$ iff there exists some $b \in B_{PK(\CKB)}$ such that $b \in
I \setminus I'$ and for all $b' <_{var} b$ it holds that $b' \in I$
iff $b' \in I'$, where $<_{var}$ is an arbitrary but fixed total order
on $B_{PK(\CKB)}$.

Intuitively, we define $\mathcal{R}_{\rm one}(\CKB)$ by the
following strategy. The domain $R$ is the set of all pairs
$(I,\chi)$, where $I$ is an interpretation of $PK(\CKB)$ and $\casmap$
a possible clashing assumption map, and two constants
$\mathbf{0}, \mathbf{1}$, which act as the zero and one of the
semiring. The multiplication $\srstimes$ of $\mathcal{R}_{\rm one}(\CKB)$ is (pointwise) union and can thus be used to build a
representation of the interpretation $I$ and its clashing assumption
map $\casmap$.  The addition $\srsplus$ corresponds to taking the
``more preferred'' interpretation or the one which is
lexicographically smaller, in case of a tie.

Note that the restriction to \eval-free sCKRs (or a similar fragment) is necessary: the strategy explained above is only viable if the preference relation over the models is transitive. 

\smallskip\noindent
\textbf{Epistemic Reasoning using Overall Weight Queries.} Using
asprin, we can enumerate preferred models.
For obtaining all of them at once, we can use an overall weight query.
\begin{theorem}
Let $\CKB$ be a single-relational, \eval-free sCKR. Then there exists a semiring $\mathcal{R}_{\rm all}(\CKB)$ and weighted formula $\alpha_{\rm all}$ such that the overall weight of $\mu_{\rm all} = \langle PK(\CKB), \alpha_{\rm all}, \mathcal{R}_{\rm all}(\CKB)\rangle$ is $(A_{\mlc})_{\mlc \in \N}$ and
the set of CKR models corresponds to
$\{(\Ic)_{\mlc \in \N} \mid \text{ for each }c \in \N: (\Ic,\casmap(\mlc)) \in A_{\mlc}\}$.
\end{theorem}
The definition of $\mathcal{R}_{\rm all}(\CKB)$ is similar to that of
$\mathcal{R}_{\rm one}(\CKB)$. However, instead of pairs $(I,
\casmap)$ the semiring values here are sets of pairs $(I,
\casmap)$. Given such sets $A,B$, 
addition and multiplication select the preferred pairs in the result
of the union $A\cup B$ and the ``Cartesian'' union $\{(S_1 \cup S_2,
\chi_1 \cup \chi_2) \mid (S_1, \chi_1) \in A, (S_2, \chi_2) \in B\}$,
respectively.

We can use the overall weight $\mu_{\rm all}(PK(\CKB))$ not only to
single out all preferred models but also for further advanced
tasks. 
E.g., the cautious and brave consequences at context
$\mlc$ 
are obtained\,by

\smallskip

\centerline{$
\bigcap \{\Ic \mid \Ic \in \mu_{\rm all}(PK(\CKB))_{\mlc}\} \quad
\text{ respectively }\quad \bigcup \{ \Ic \mid \Ic \in \mu_{\rm all}(PK(\CKB))_{\mlc}\}.  
$}

\smallskip

Apart from this, we can also use the result to evaluate epistemic
aggregate queries, akin to the ones defined by
\citeN{calvanese2008aggregate}, 
of the form
\[
q(\overline{x},\alpha(\overline{y})) \leftarrow \mathbf{K}\; \overline{x}, \overline{y}, \overline{z}. \phi,[ \psi ],
\]
where $\phi$ and $\psi$ are conjunctions of possibly non-ground atoms
and $\overline{x}, \overline{y}, \overline{z}$ are sequences of
variables that occur in $\phi$, such that $\overline{z}$ is distinct
from $\overline{x}$ and $\overline{y}$. Furthermore, $\alpha$ is an
aggregation function. The meaning of this expression given a knowledge
base $\mathcal{KB}$ is intuitively as follows. For each assignment to
$\overline{x}$, we aggregate over all values
$\overline{y}$ using 
$\alpha$, subject to the
constraint that for every model $D$ of $\mathcal{KB}$ the assignment
to $\overline{x}, \overline{y}$ can be completed to an assignment
$\gamma$ to all the variables in $\phi$ and $\psi$ such that (i)
$\phi$ and $\psi$ are satisfied by $D$ w.r.t.\ $\gamma$, and (ii)
for every
model $D'$ of $\mathcal{KB}$ it holds that $\gamma$ restricted to
$\overline{x},\overline{y}, \overline{z}$ is a certain answer for the
query
($*$) $aux_q(\overline{x}, \overline{y}, \overline{z}) \leftarrow \phi, \psi$.
Then, $q(t, z)$ is an answer of the above epistemic aggregate query if
it is the result of the query in every model $D$. For formal details, we
refer to the paper by \citeN{calvanese2008aggregate}.

\citeANP{calvanese2008aggregate} showed that for ``restricted''
queries, the value of the aggregate is obtained by
\begin{align*}
	q_0(\overline{x}, \overline{y}, \overline{z}^{\phi})\leftarrow \text{Cert}(aux_q,\mathbf{K})(\overline{x}, \overline{y}, \overline{z}).\qquad
	q_1(\overline{x}, \alpha(\overline{y}))\leftarrow q_0(\overline{x}, \overline{y}, \overline{z}^{\phi}).
\end{align*}
Here $\overline{z}^{\phi}$ are the variables of $\overline{z}$ that occur in $\phi$ and $ \text{Cert}(aux_q,\mathbf{K})(\overline{x}, \overline{y}, \overline{z})$ refers to the certain answers of the query
($*$).
Unfortunately, we cannot use ASP
alone to compute the certain answers in the presence of defeasible
axioms and preferences in sCKRs. However, the overall weight $\mu^*(PK(\CKB))$
contains the information necessary to conclude what the certain
answers are. These in turn can then be used to evaluate epistemic
aggregates over sCKRs.

\section{Discussion and Conclusions}
\label{sec:conclusions}

We considered the application of 
ASP with algebraic measures for expressing preferences of defeasibility
in multi-relational CKRs.
%
The problem of representing notions of defeasibility in DLs
has led to many proposals 
and is still an active area of research~\cite{GiordanoGOP:11,BonattiFPS:15,DBLP:journals/ijar/PenselT18,DBLP:journals/tocl/BritzCMMSV21}.
A detailed comparison of 
justifiable exceptions
with other definitions of non-monotonicity in DLs and contextual systems can
be found in the papers by~\citeANP{BozzatoES:18} \shortcite{BozzatoES:18,BozzatoES:CONTEXT19}.
\nop{*******
We note that other works that add forms of defeasibility in DLs
take advantage of the non-monotonic features of 
semantics of logic programs: for example, 
ASP with stratified negation was used in~\cite{GiordanoD:18}
for reasoning on the rational closure of 
${\cal SROEL}(\sqcap, \times)$; moreover, the 
addition of non-monotonic inference to DLs was a 
motivation 
for bridging DLs with ASP
like e.g in dl-programs~\cite{EiterKSX:12}.
*********}
Our work on CKRs with multiple contextual relations
was influenced by approaches dealing with exceptions under different relations 
or diverse definitions of normality. One of the latest in this direction 
is the work by \citeN{GiordanoD:20TPLP},
where the notion of typicality in DLs is extended
to a ``concept-aware multi-preference semantics'':
the domain elements are organized in multiple preference orderings 
$\leq_C$ to represent their typicality w.r.t.\ a concept $C$;
models are then ordered by a global preference combining the concept-related preferences.
Similar to our approach, entailment is encoded in ASP 
using a fragment of  
\citeANP{Krotzsch:10}'s \citeyear{Krotzsch:10} materialization calculus and 
representing combination of preferences in asprin.
%
\citeN{Gil:14}
earlier studied the effects of adding multiple preferences to a typicality extension of ${\cal ALC}$.

Concerning semirings for general quantitative specifications,
several works
used semirings to define quantitative generalisations of well-known qualitative problems. 
For example, \emph{Semiring-based Constraint Satisfaction Problems (SCSP)}~\cite{bistarelli1999semiringJOURNAL} 
allow for quantitative semantics of CSP's and capture other quantitative extensions of CSP's (weighted CSP) as special cases for some specific semiring. 
Semiring Provenance~\cite{green2007provenance}, generalizes the bag
semantics and other definitions of provenance for relational algebra
to semirings: this allows one to 
capture existing quantitative semantics, but also to introduce additional novel capabilities 
to obtain the provenance lineage of a query.
%
Moreover, algebraic ProbLog~\cite{kimmig2011algebraic} introduced an
algebraic semantics of logic programs by facilitating semirings. 
Intuitively, their approach
can be seen as a fragment of ASP with algebraic measures
allowing only a restricted use of negation in programs and no arbitrary 
recursive sums and products in the weighted formulas.
%

The parametrization of semantics with a semiring allows for flexible and highly general quantitative frameworks: in particular, algebraic measures 
allow for an intuitive specification of computations depending on the
    truth of propositional variables.
Building on ASP, 
they offer an appealing specification language for quantitative reasoning problems like preferential reasoning.


\smallskip\noindent
\textbf{Outlook.} In the direction of using the capabilities of algebraic measures
for comparing models, we plan to further study the possibilities
for epistemic reasoning on DLs as introduced in previous sections.
With respect to contextual reasoning, a possible continuation
of this work can consider a refinement of the 
definitions of preference and knowledge propagation across different 
contextual relations,
possibly by considering a motivating real-world application.

\bigskip\noindent
\textbf{Acknowledgments.}
This work was partially supported by the European Commission funded
projects ``Humane AI: Toward AI Systems That Augment and Empower Humans by
Understanding Us, our Society and the World Around Us'' (grant \#820437) and
``AI4EU: A European AI on Demand Platform and Ecosystem'' (grant \#825619),
and the Austrian Science Fund (FWF) project W1255-N23. The support is
gratefully acknowledged.


\bibliographystyle{acmtrans}

\newpage
\appendix
\section{Single-relational Example}
We also give an example of a single-relational sCKR.
\begin{example}
  \label{ex:coverage-only}
  We consider a
	single-relation hierarchy on coverage
	by reviewing the example from~\cite{BozzatoES:CONTEXT19,DBLP:conf/kr/BozzatoSE18}.
	%
  Let us consider a sCKR $\CKB_{org1} = \stru{\Cl, \KB_{\N}}$ 
	with with $\Cl = (\N, \prec_c)$ describing the organization of a corporation.
  The corporation wants to define different policies with respect to 
  its local branches, represented by the 
	coverage hierarchy in $\Cl$. 
%
  %
  The corporation is active in the fields of
  Musical instruments ($\mi{M}$), Electronics ($\mi{E}$) and Robotics ($\mi{R}$). 
  A supervisor ($\mi{S}$) can be assigned to manage only one of these fields.
	Defeasible axioms in contexts in $\KB_{\N}$ define the 
	assignment of local supervisors to their field:
	\begin{center}
		$\begin{array}{r@{\;}l}
			\Cl : & \{ \mlc_\mi{\branch1} \precc \mlc_\mi{world},\
			         \mlc_\mi{\branch2} \precc \mlc_\mi{world},\
							 \mlc_\mi{\branch1} \precc \mlc_\mi{\branch2},\
							 \mlc_\mi{\local1} \precc \mlc_\mi{\branch1},\}\\[1ex] 
			\mlc_\mi{world} :& \{ \mi{M} \sqcap \mi{E} \subs \bot,\ 
			                        \mi{M} \sqcap \mi{R} \subs \bot,\
															\mi{E} \sqcap \mi{R} \subs \bot,\
										   	\default(\mi{S} \subs \mi{E}) \}\\[.5ex]
			\mlc_\mi{\branch1}    :& \{ \default_c(\mi{S} \subs \mi{M}) \} \qquad
			\mlc_\mi{\branch2}    : \{ \default_c(\mi{S} \subs \mi{R}) \} \qquad
			\mlc_\mi{\local1}     : \{ \mi{S}(i) \}
		\end{array}$
	\end{center}
  In $\mlc_\mi{world}$ we say that supervisors are assigned to Electronics, 
  while in the sub-context for $\mlc_\mi{\branch2}$ we contradict this by assigning 
  all local supervisors to the Robotics area and in $\mlc_\mi{\branch1}$
	we further specialize this by assigning supervisors to the Musical instruments area. 
  In the context $\mlc_\mi{\local1}$ for a local site we have information about an instance $i$.
	Note that different assignments of areas for $i$ are possible 
	by instantiating the defeasible axioms: 
	intuitively,
	we want to prefer the interpretations that override
  the higher defeasible axioms in $\mlc_\mi{world}$ and $\mlc_\mi{\branch2}$.
  
	Observe that different 
	justified $\CAS$ models are possible, depending on 
  the different assignments of the individual $i$ in $\mlc_\mi{local1}$
	to the alternative areas denoted by defeasible axioms. 
	We have three possible 
	clashing assumptions sets for context $\mlc_\mi{local1}$:
	\begin{center}\small
		$\begin{array}{l}
	    \chi_{c}^{1}(\mlc_\mi{local1}) = \{\stru{S \subs E, i}, \stru{S \subs R, i}\} \quad
	    \chi_{c}^{2}(\mlc_\mi{local1}) = \{\stru{S \subs M, i}, \stru{S \subs R, i}\} \\[.5ex]
			\chi_{c}^{3}(\mlc_\mi{local1}) = \{\stru{S \subs M, i}, \stru{S \subs E, i}\}
    \end{array}$
	\end{center}
	By the ordering on clashing assumption sets,
	in particular 
	$\chi_{c}^{1}(\mlc_\mi{\local1}) > \chi_{c}^{2}(\mlc_\mi{\local1})$,
	$\chi_{c}^{1}(\mlc_\mi{\local1}) > \chi_{c}^{3}(\mlc_\mi{\local1})$ and
	$\chi_{c}^{3}(\mlc_\mi{\local1}) > \chi_{c}^{2}(\mlc_\mi{\local1})$.
	Thus, 
        $\CKB_{org1}$ has one preferred model
	which corresponds to $\chi_{c}^{1}$: 
	it corresponds to the intended 
	interpretation in which the defeasible axiom $\default(S \subs M)$ 
	associated to 
	$\mlc_\mi{\branch1}$ wins 
	over the more general rules asserted in 
	$\mlc_\mi{\branch2}$ and $\mlc_\mi{world}$.\EndEx
\end{example}

\section{ASP Translation and Rule Set Tables}

We provide further details on the ASP encoding introduced in Section~\ref{sec:asprin_enc}.
The ASP translation is defined by adapting the encoding presented in~\cite{BozzatoES:CONTEXT19,DBLP:conf/kr/BozzatoSE18} (which, in turn, is based on the translation introduced in~\cite{BozzatoES:18})
to the manage the interpretation of multiple relations in simple CKRs.

The ASP translation is defined for $\SROIQrld$ multi-relational simple CKRs of the form 
$\CKB = \stru{\Cl, \KB_{\N}}$ with $\Cl = (\N, \prec_t, \prec_c)$, i.e. 
over time and coverage contextual relations.

The language of $\SROIQrld$~\cite{BozzatoES:18} restrict the form of $\SROIQrl$ expressions 
in defeasible axioms:
in defeasible axioms, $D \sqcap D$ can not appear as a right-side concept and each right-side concept
$\forall R.D$ has $D \in \NC$.
%
We consider the $\SROIQrld$ 
normal form transformation proposed in~\cite{BozzatoES:18}  
for the formulation of the rules
(considering axioms that can appear in simple CKRs)
and we assume again the Unique Name Assumption.
For ease of reference, the form of (strict and defeasible) axioms in normal form
is presented in Table~\ref{tab:normalform}.
Note that we further simplified the normalization of defeasible 
class and role assertions and negative assertions 
as they can be easily represented
using defeasible class and role inclusions with auxiliary symbols.
%
\begin{table}[t!]%
\caption{$\SROIQrld$ normal form for axioms in $\Lcal_\Sigma$}
\label{tab:normalform}
\vspace{-1ex}
\centerline{\small
$\begin{array}{c}
\hline\\[-2ex]
\multicolumn{1}{l}{\text{\textbf{Strict axioms:} for $A,B \in \NC$, $R,S,T \in \NR$, 
$a,b \in \NI$, $\mlc \in \N$:}}\\[1ex]
\begin{array}{l}
  A(a) \qquad  R(a,b) \qquad a = b \qquad a \neq b\\[1ex]
	A  \subs B \qquad  \{a\} \subs B \qquad A \sqcap B \subs C \\[1ex]
	\exists R.A \subs B \qquad A \subs \exists R.\{a\} \qquad A \subs \forall R.B \qquad 
	A \subs {\leqslant} 1 R.\top\\[1ex]
  R  \subs T \qquad R \circ S \subs T \qquad \mathrm{Dis}(R,S) \qquad \mathrm{Inv}(R,S) \qquad
  \mathrm{Irr}(R)\\[1ex]
	\eval(A, \mlc) \subs B \qquad \eval(R, \mlc) \subs S\\[.5ex]
\end{array}\\
\hline\\[-2ex]
\multicolumn{1}{l}{\text{\textbf{Defeasible axioms:} for $A,B \in \NC$, $R,S \in \NR$, $a \in \NI$, 
$rel \in \{t,c\}$:}}\\[1ex]
\begin{array}{l}
  \default_{rel}(A \subs B) \qquad \default_{rel}(A \sqcap B \subs C) \qquad
  \default_{rel}(\exists R.A \subs B) \\[1ex]
	\default_{rel}(A \subs \exists R.\{a\}) \qquad
	\default_{rel}(A \subs \forall R.B) \qquad \default_{rel}(A \subs {\leqslant} 1 R.\top)\\[1ex]	
	\default_{rel}(R \subs S) \quad \default_{rel}(R \circ S \subs T) \quad \default_{rel}(\mathrm{Dis}(R,S))
	\quad \default_{rel}(\mathrm{Inv}(R,S)) \quad \default_{rel}(\mathrm{Irr}(R))\\[.5ex]
\end{array}\\ 
\hline
\end{array}$}
\end{table}

As in the original formulation (inspired by the materialization calculus in~\cite{Krotzsch:10}),
the translation includes sets of \emph{input rules $I$} (which encode DL axioms and signature as facts),
\emph{deduction rules $P$} (normal rules providing instance level inference) and \emph{output rules $O$}
(that encode in terms of a fact the ABox assertion to be proved).

The sets of rules for the proposed translation are presented in tables in the following
pages.
The input rules $I_{rl}$ and deduction rules $P_{rl}$ for
$\SROIQrl$ axioms are shown in Table~\ref{tab:rl-rules-tgl}.
Table~\ref{tab:global-local-rules-tgl} shows 
input rules $I_{glob}$ and deduction rules $P_{glob}$
for the translation of the contextual structure in $\Cl$
local input rules $I_{eval}$ and deduction rules $P_{eval}$
for managing $\eval$ expressions, and
output rules $O$ for encoding the output instance query.
Input rules $I_\default$ in Table~\ref{tab:input-default-tgl}
provide the encoding of defeasible axioms.
Deduction rules in $P_\default$ manage the interpretation
of defeasible axioms and knowledge propagation. 
Table~\ref{tab:ovr-rules-tgl} shows
rules defining the overriding of axioms.
Rules for the inheritance of strict axioms are shown in Table~\ref{tab:strict-inheritance-rules-tgl},
while rules in Table~\ref{tab:def-inheritance-rules-tgl}
define defeasible inheritance. Table~\ref{tab:par-inheritance-rules-tgl}
shows rules for the propagation of defeasible axioms on a relation $\mi{rel1}$
over the other relation. Auxiliary test rules in $P_\default$ are shown 
in Table~\ref{tab:test-rules-tgl}.
Finally, rules and directives in $P_{pref}$ define the 
asprin preference: the definition of asprin local and global
preferences is shown in Table~\ref{tab:preference-tbl},
while rules in Table~\ref{tab:prep-rules-tgl}
provide auxiliary rules.

Given a multi-relational sCKR $\CKB = \stru{\Cl, \KB_{\N}}$ in $\SROIQrld$ normal form
with $\Cl = (\N, \prec_t, \prec_c)$, 
a program $PK(\CKB)$ that encodes 
$\CKB$ is obtained as follows:
\begin{enumerate}
\item 
  the \emph{global program} for $\Cl$ is built as: 
	  $PG(\Cl) = I_{glob}(\Cl) \cup P_{glob}$
\item 
  for each $\mlc \in \N$, we define each local program for context $\mlc$ as:
	$PC(\mlc, \CKB) = I_{loc}(\KB_\mlc, \mlc) \cup P_{loc}$, where
	$I_{loc}(\KB_\mlc, \mlc) = 
	 I_{rl}(\KB_\mlc, \mlc) \cup I_{eval}(\KB_\mlc, \mlc) \cup I_{\default}(\KB_\mlc, \mlc)$
	and 
	$P_{loc} = P_{rl} \cup P_{eval} \cup P_{\default}$
\item 
  The \emph{CKR program} $PK(\CKB)$ is defined as:
	  $PK(\CKB) = PG(\Cl) \cup \bigcup_{\mlc \in \N} PC(\mlc, \CKB)$
\end{enumerate}%
Query answering $\CKB \models \mlc: \alpha$ is then obtained 
by testing whether the instance query, translated to
ASP by $O(\alpha, \mlc)$, is a consequence of 
the \emph{preferred} models of $PK(\CKB)$, 
i.e., whether $PK(\CKB) \cup P_{pref} \models O(\alpha, \mlc)$ holds. 
This can be extended to conjunctive queries $Q$ 
by applying the output rules to its atoms and checking if
$PK(\CKB) \cup P_{pref} \models O(Q)$ holds.

\begin{table}[tp!]%
\caption{$\SROIQrl$ input and deduction rules}
\vspace{2ex}
\hrule\mbox{}\\ 
\textbf{$\SROIQrl$ input translation $I_{rl}(S,c)$}\\[.5ex]
\scalebox{.9}{
\small
$\begin{array}[t]{l@{\ \ }l}               
\mbox{(irl-nom)} 
& a \in \NI \mapsto \{\nom(a,c)\}\\
\mbox{(irl-cls)} 
& A \in \NC \mapsto \{\cls(A,c)\}\\
\mbox{(irl-rol)} 
& R \in \NR \mapsto \{\rol(R,c)\}\\[1ex]

\mbox{(irl-inst1)} 
& A(a) \mapsto \{\insta(a,A,c)\} \\
\mbox{(irl-triple)} 
& R(a,b) \mapsto \{\triplea(a,R,b,c)\} \\
\mbox{(irl-eq)} & a = b \mapsto \{\peq(a,b,c,\ml{main})\} \\
\mbox{(irl-neq)} & a \neq b \mapsto \emptyset \\
\mbox{(irl-inst3)} 
& \{a\} \subs B \mapsto \{\insta(a,B,c)\} \\
\mbox{(irl-subc)} 
& A \subs B \mapsto \{\subClass(A,B,c)\} \\
\mbox{(irl-top)} & \top(a) \mapsto \{\insta(a,\ptop,c)\} \\
\mbox{(irl-bot)} & \bot(a) \mapsto \{\insta(a,\pbot,c)\}

\end{array}$
\;
$\begin{array}[t]{l@{\ \ }l}               
\mbox{(irl-subcnj)} 
& A_1 \sqcap A_2 \subs B \mapsto \{\subConj(A_1,A_2,B,c)\} \\
\mbox{(irl-subex)} 
& \exists R.A \subs B \mapsto \{\subEx(R,A,B,c)\} \\[1ex]

\mbox{(irl-supex)} &  A \subs \exists R.\{a\} \mapsto \{\supEx(A,R,a,c)\}\\
\mbox{(irl-forall)} &  A \subs \forall R.B \mapsto \{\supForall(A,R,B,c)\} \\
\mbox{(irl-leqone)} &  A \subs {\leqslant} 1 R.\top \mapsto \{\supLeqOne(A,R,c)\} \\[1ex]
        
\mbox{(irl-subr)} 
& R \subs S \mapsto \{\subRole(R,S,c)\}\\
\mbox{(irl-subrc)} 
& R {\circ} S \subs T \mapsto \{\subRChain(R,S,T,c)\}\\
\mbox{(irl-dis)} & \mathrm{Dis}(R,S)  \mapsto \{\pDis(R,S,c)\}\\
\mbox{(irl-inv)} & \mathrm{Inv}(R,S) \mapsto \{\pInv(R,S,c)\}\\
\mbox{(irl-irr)} & \mathrm{Irr}(R) \mapsto \{\pIrr(R,c)\}\\
  \end{array}$}\\[2ex]
  
\textbf{$\SROIQrl$ deduction rules $P_{rl}$}\\[.5ex]
\scalebox{.9}{
\small
$\begin{array}{l@{\;}r@{\ }r@{\ }l@{}}
	 \mbox{(prl-instd)} & \instd(x,z,c,\main) & \rif & \insta(x,z,c).\\
	 \mbox{(prl-tripled)} & \tripled(x,r,y,c,\main) & \rif & \triplea(x,r,y,c).\\[1ex]


   \mbox{(prl-eq)} & \unsat(t) & \rif & \peq(x,y,c,t).\\
   \mbox{(prl-top)} & \instd(x,\ptop,c,\ml{main}) & \rif & \nom(x,c).\\	
   \mbox{(prl-bot)} & \unsat(t) & \rif & \instd(x,\pbot,c,t).\\[1ex]
		
   \mbox{(prl-subc)} 
   &      \instd(x,z,c,t) & \rif & \subClass(y,z,c), \instd(x,y,c,t). \\
   \mbox{(prl-subcnj}) 
   & \instd(x,z,c,t) & \rif & \subConj(y_1,y_2,z,c), \instd(x,y_1,c,t),\instd(x,y_2,c,t). \\
   \mbox{(prl-subex)} 
   & \instd(x,z,c,t) & \rif & \subEx(v,y,z,c), \tripled(x,v,x',c,t),\instd(x',y,c,t). \\
   \mbox{(prl-supex)} 
   & \tripled(x,r,x',c,t) & \rif & \supEx(y,r,x',c), \instd(x,y,c,t). \\
   \mbox{(prl-supforall)} 
   & \instd(y,z',c,t) & \rif & \supForall(z,r,z',c), \instd(x,z,c,t), \tripled(x,r,y,c,t).\\
   \mbox{(prl-leqone)} & \unsat(t) & \rif & \supLeqOne(z,r,c), \instd(x,z,c,t),\\
           & & & \tripled(x,r,x_1,c,t),\tripled(x,r,x_2,c,t).\\[1ex]					

   \mbox{(prl-subr)} 
   & \tripled(x,w,x',c,t) & \rif & \subRole(v,w,c), \tripled(x,v,x',c,t). \\
   \mbox{(prl-subrc)} 
   & \tripled(x,w,z,c,t) & \rif & \subRChain(u,v,w,c), \tripled(x,u,y,c,t), \tripled(y,v,z,c,t).\\[1ex]

   \mbox{(prl-dis)} & \unsat(t) & \rif & \pDis(u,v,c), \tripled(x,u,y,c,t), \tripled(x,v,y,c,t).\\	
   \mbox{(prl-inv1)} & \tripled(y,v,x,c,t) & \rif & \pInv(u,v,c), \tripled(x,u,y,c,t). \\
   \mbox{(prl-inv2)} & \tripled(y,u,x,c,t) & \rif & \pInv(u,v,c), \tripled(x,v,y,c,t). \\
   \mbox{(prl-irr)} & \unsat(t) & \rif & \pIrr(u,c),\tripled(x,u,x,c,t).\\[1ex]
   \mbox{(prl-sat)} &   & \rif & \unsat(\main).
  \end{array}$}\\[.5ex]
\hrule\mbox{}
\label{tab:rl-rules-tgl}
\end{table}

\begin{table}[tp]%
\caption{Global, local and output rules}
\label{tab:global-local-rules-tgl}

\medskip

\hrule\mbox{}\\
\textbf{Global input rules $I_{glob}(\Cl)$}\\[.5ex]
\scalebox{.9}{
\small
$\begin{array}{l@{\ \ }r@{\;}l}

 \mbox{(igl-ctx)}   & \ml{c} \in \N   &\mapsto \{\pcontext(\mlc)\}\\
 \mbox{(igl-rel-t)} & \prec_t \in \Cl &\mapsto \{\prelation(\ml{time})\}\\
 \mbox{(igl-rel-c)} & \prec_c \in \Cl &\mapsto \{\prelation(\ml{covers})\}\\[1ex]

 \mbox{(igl-covers-t)} & \mlc_1 \prec_t \mlc_2 &\mapsto \{\pprec(\mlc_1, \mlc_2, \ml{time})\}\\
 \mbox{(igl-covers-c)} & \mlc_1 \prec_c \mlc_2 &\mapsto \{\pprec(\mlc_1, \mlc_2, \ml{covers})\}\\[1ex]
\end{array}$}

\textbf{Global deduction rules $P_{glob}$}\\[.5ex]
\scalebox{.9}{
\small
$\begin{array}{l@{\ \ }r@{\;}l}               
\mbox{(pgl-preceq1)} & 
\ppreceq(c_1, c_2, rel) \rif & \pprec(c_1, c_2, rel).\\
\mbox{(pgl-preceq2)} & 
\ppreceq(c_1, c_1, rel) \rif & \pcontext(c_1), \prelation(rel).\\[1ex]

\mbox{(pgl-preceqexc1)} & 
\preceqex(c_1, c_2, rel) \rif & \prelation(rel), \ppreceq(c_1, c_3, rel_1),\\
                             && \ppreceq(c_3, c_2, rel_2), rel \neq rel_1, rel \neq rel_2.\\
\mbox{(pgl-preceqexc2)} & 
\preceqex(c_1, c_2, rel) \rif & \prelation(rel), \ppreceq(c_1, c_2, rel_1), rel \neq rel_1.\\[2ex]

\end{array}$}

\textbf{Local \eval\ input rules $I_{eval}(S, c)$}\\[.5ex]
\scalebox{.9}{
\small
$\begin{array}{l@{\ \ }l}
\mbox{(ilc-subevalat)} & \eval(A, \ml{c}_1) \subs B \mapsto \{\subEval(A, \ml{c}_1, B, \mlc)\} \\
\mbox{(ilc-subevalr)} & \eval(R, \ml{c}_1) \subs T \mapsto \{\subEvalR(R, \ml{c}_1, T, \mlc)\}\\[1ex]
\end{array}$}

\textbf{Local \eval\ deduction rules $P_{eval}$}\\[.5ex]
\scalebox{.9}{
\small
$\begin{array}{l@{\ \ }r@{\;}l}               
\mbox{(plc-subevalat)} & 
\instd(x, b, c, t) \rif & \subEval(a, c_1, b, c), \instd(x, a, c_1, t).\\
\mbox{(plc-subevalr)} & 
\tripled(x, s, y, c, t) \rif & \subEvalR(r, c_1, s, c), \tripled(x, r, y, c_1, t).\\[1ex]

\mbox{(plc-subevalatp)} & 
\instd(x, b, c, t) \rif & \subEval(a, c_1, b, c_2), \instd(x, a, c_1, t),\\ 
&& \pprec(c,c_3, rel1), \ppreceq(c_3,c_2, rel2), rel1 \neq rel2.\\
\mbox{(plc-subevalrp)} & 
\tripled(x, s, y, c, t) \rif & \subEvalR(r, c_1, s, c_2), \tripled(x, r, y, c_1, t),\\ 
&& \pprec(c,c_3, rel1), \ppreceq(c_3,c_2, rel2), rel1 \neq rel2.\\[2ex]
\end{array}$}

\textbf{Output translation $O(\alpha,\mlc)$}\\[.5ex]
\scalebox{.9}{
\small
$\begin{array}{l@{\ \ }l}
\mbox{(o-concept)} & A(a) \mapsto \{\instd(a,A,\mlc,\main)\} \\
\mbox{(o-role)} & R(a,b) \mapsto \{\tripled(a,R,b,\mlc,\main)\} \\[1ex]
\end{array}$}
\hrule
\end{table}

\begin{table}[tp]%
\caption{Input rules $I_{\default}(S,c)$ for defeasible axioms}
\label{tab:input-default-tgl} 

\medskip

\hrule\mbox{}\\[1ex]
{
\small
\scalebox{.9}{
$\begin{array}{@{}l@{~}r@{~}l@{}}

 \mbox{(id-subc)} & \default_{rel}(A \subs B)  & \mapsto \{\, \defsubs(A,B,c,rel).\,\} \\  
 \mbox{(id-subcnj)} & \default_{rel}(A_1 \sqcap A_2 \subs B)& \mapsto\{\,\defsubcnj(A_1, A_2, B,c,rel).\,\}\\  
 \mbox{(id-subex)} & \default_{rel}(\exists R.A \subs B)  & \mapsto \{\, \defsubex(R, A, B,c,rel).\,\} \\    
 \mbox{(id-supex)} & \default_{rel}(A \subs \exists R.\{a\})  & \mapsto \{\, \defsupex(A, R, a,c,rel).\,\} \\   
 \mbox{(id-forall)} & \default_{rel}(A \subs \forall R.B)  & \mapsto \{\, \defsupforall(A, R, B,c,rel).\,\} \\    
 \mbox{(id-leqone)} & \default_{rel}(A \subs {\leqslant} 1 R.\top)  & \mapsto \{\, \defsupleqone(A, R,c,rel).\,\} \\[1ex]
 \mbox{(id-subr)} & \default_{rel}(R \subs S)  & \mapsto \{\, \defsubr(R, S,c,rel).\,\} \\    
 \mbox{(id-subrc)} & \default_{rel}(R \circ S \subs T)  & \mapsto \{\, \defsubrc(A_1, A_2, B,c,rel).\,\} \\    
 \mbox{(id-dis)} & \default_{rel}(\mathrm{Dis}(R,S))  & \mapsto \{\, \defdis(R, S,c,rel).\,\} \\   
 \mbox{(id-inv)} & \default_{rel}(\mathrm{Inv}(R,S))  & \mapsto \{\, \definv(R,S,c,rel).\,\} \\   
 \mbox{(id-irr)} & \default_{rel}(\mathrm{Irr}(R))  & \mapsto \{\, \defirr(R,c,rel).\,\} \\    
\end{array}$}}\\[1ex]
\hrule
\end{table}

\begin{table}[tp]%
\caption{Deduction rules $P_{\default}$ for defeasible axioms: overriding rules}
\label{tab:ovr-rules-tgl}

\medskip

\hrule\mbox{}\\[1ex]
\scalebox{.9}{
$\begin{array}{l@{\ \ }r@{\ \ }l}
 \mbox{(ovr-subc)} & 
 \ovr(\subClass,x,y,z,c_1,c,rel1) \rif &
 \defsubs(y,z,c_1,rel1),\\
  && \pprec(c,c_2,rel1), \ppreceq(c_2,c_1,rel2), rel1 \neq rel2,\\
	&& \instd(x,y,c,\ml{main}), \naf \testf(\nlit(x,z,c)). \\[0.5ex]
  \mbox{(ovr-cnj)} & 
  \ovr(\subConj,x,y_1,y_2,z,c_1,c,rel1) \rif &  
	\defsubcnj(y_1,y_2,z,c_1,rel1),\\ 
	&& \pprec(c,c_2,rel1), \ppreceq(c_2,c_1,rel2), rel1 \neq rel2,\\  
	&& \instd(x,y_1,c,\ml{main}), \instd(x,y_2,c,\ml{main}),\\ 
	&& \naf \testf(\nlit(x,z,c)).\\[0.5ex]
  \mbox{(ovr-subex)} & 
  \ovr(\subEx,x,r,y,z,c_1,c,rel1) \rif &  
	\defsubex(r,y,z,c_1,rel1), \\
	&& \pprec(c,c_2,rel1), \ppreceq(c_2,c_1,rel2), rel1 \neq rel2,\\  
	&& \tripled(x,r,w,c,\ml{main}), \instd(w,y,c,\ml{main}),\\ 
	&& \naf \testf(\nlit(x,z,c)).\\[0.5ex]
  \mbox{(ovr-supex)} & 
  \ovr(\supEx,x,y,r,w,c_1,c,rel1) \rif &  
	\defsupex(y,r,w,c_1,rel1),\\ 
	&& \pprec(c,c_2,rel1), \ppreceq(c_2,c_1,rel2), rel1 \neq rel2,\\ 
  && \instd(x,y,c,\ml{main}), \naf \testf(\nrel(x,r,w,c)).\\[0.5ex]
  \mbox{(ovr-forall)} & 
  \ovr(\supForall,x,y,z,r,w,c_1,c,rel1) \rif & 
	\defsupforall(z,r,w,c_1,rel1),\\
  && \pprec(c,c_2,rel1), \ppreceq(c_2,c_1,rel2), rel1 \neq rel2,\\ 
	&& \instd(x,z,c,\ml{main}), \tripled(x,r,y,c,\ml{main}),\\
	&& \naf \testf(\nlit(y,w,c)). \\[0.5ex]
  \mbox{(ovr-leqone)} & 
  \ovr(\supLeqOne,x,x_1,x_2,z,r,c_1,c,rel1) \rif &
  \defsupleqone(z,r,c_1,rel1),\\ 
	&& \pprec(c,c_2,rel1), \ppreceq(c_2,c_1,rel2), rel1 \neq rel2,\\ 
	&& \instd(x,z,c,\ml{main}), \tripled(x,r,x_1,c,\ml{main}),\\ 
	&& \tripled(x,r,x_2,c,\ml{main}),\\[0.5ex]
  \mbox{(ovr-subr)}  & 
  \ovr(\subRole,x,y,r,s,c_1,c,rel1) \rif &
  \defsubr(r,s,c_1,rel1), \\
	&& \pprec(c,c_2,rel1), \ppreceq(c_2,c_1,rel2), rel1 \neq rel2,\\ 
	&& \tripled(x,r,y,c,\ml{main}), \naf \testf(\nrel(x,s,y,c)).\\[0.5ex]                     
  \mbox{(ovr-subrc)}  & 
  \ovr(\subRChain,x,y,z,r,s,t,c_1,c,rel1) \rif &   
	\defsubrc(r,s,t,c_1,rel1),\\
  && \pprec(c,c_2,rel1), \ppreceq(c_2,c_1,rel2), rel1 \neq rel2,\\ 
	&& \tripled(x,r,y,c,\ml{main}), \tripled(y,s,z,c,\ml{main}),\\  
	&& \naf \testf(\nrel(x,t,z,c)).\\[0.5ex]                     
  \mbox{(ovr-dis)}  & 
  \ovr(\pDis,x,y,r,s,c_1,c,rel1) \rif &
  \defdis(r,s,c_1,rel1),\\ 
	&& \pprec(c,c_2,rel1), \ppreceq(c_2,c_1,rel2), rel1 \neq rel2,\\ 
	&& \tripled(x,r,y,c,\ml{main}), \tripled(x,s,y,c,\ml{main}).\\
  \mbox{(ovr-inv1)}  & 
  \ovr(\pInv,x,y,r,s,c_1,c,rel1) \rif &
  \definv(r,s,c_1,rel1),\\ 
	&& \pprec(c,c_2,rel1), \ppreceq(c_2,c_1,rel2), rel1 \neq rel2,\\ 
	&& \tripled(x,r,y,c,\ml{main}), \naf \testf(\nrel(x,s,y,c)).\\[0.5ex]
  \mbox{(ovr-inv2)} & 
	\ovr(\pInv,x,y,r,s,c_1,c,rel1) \rif &
  \definv(r,s,c_1,rel1),\\ 
	&& \pprec(c,c_2,rel1), \ppreceq(c_2,c_1,rel2), rel1 \neq rel2,\\ 
	&& \tripled(y,s,x,c,\ml{main}), \naf \testf(\nrel(x,r,y,c)).\\[0.5ex]
  \mbox{(ovr-irr)}  & 
  \ovr(\pIrr,x,r,c_1,c,rel1) \rif &
  \defirr(r,c_1,rel1),\\ 
	&& \pprec(c,c_2,rel1), \ppreceq(c_2,c_1,rel2), rel1 \neq rel2,\\ 
	&& \tripled(x,r,x,c,\ml{main}).\\[1.5ex]
\end{array}$}
\hrule
\end{table}

\begin{table}[tp]%

\bigskip

\caption{Deduction rules $P_{\default}$ for defeasible axioms: strict inheritance rules}
\label{tab:strict-inheritance-rules-tgl}

\medskip

\hrule\mbox{}\\[1ex]
\scalebox{.9}{
\small
$\begin{array}{l@{\;}r@{\ }r@{\ }l@{}}
   \mbox{(props-inst)} 
   &   \instd(x,z,c,\main) & \rif & 
	 \insta(x,z,c_1),\\ 
	 &&& \pprec(c,c_2, rel1), \ppreceq(c_2,c_1, rel2), rel1 \neq rel2.\\[0.5ex]
   \mbox{(props-triple)} 
   &   \tripled(x,r,y,c,\main) & \rif & 
	 \triplea(x,r,y,c_1),\\ 
   &&& \pprec(c,c_2, rel1), \ppreceq(c_2,c_1, rel2), rel1 \neq rel2.\\[0.5ex]
   \mbox{(props-subc)} 
   &    \instd(x,z,c,t) & \rif & 
	 \subClass(y,z,c_1), \instd(x,y,c,t),\\
   &&& \pprec(c,c_2, rel1), \ppreceq(c_2,c_1, rel2), rel1 \neq rel2.\\[0.5ex]
   \mbox{(props-cnj}) 
   & \instd(x,z,c,t) & \rif & 
	 \subConj(y_1,y_2,z,c_1), \instd(x,y_1,c,t), \instd(x,y_2,c,t),\\
   &&& \pprec(c,c_2, rel1), \ppreceq(c_2,c_1, rel2), rel1 \neq rel2.\\[0.5ex]
   \mbox{(props-subex)} 
   & \instd(x,z,c,t) & \rif & 
	 \subEx(v,y,z,c_1), \tripled(x,v,x',c,t),\instd(x',y,c,t),\\
   &&& \pprec(c,c_2, rel1), \ppreceq(c_2,c_1, rel2), rel1 \neq rel2.\\[0.5ex]
   \mbox{(props-supex)} 
   & \tripled(x,r,x',c,t) & \rif & 
	 \supEx(y,r,x',c_1), \instd(x,y,c,t),\\
	 &&& \pprec(c,c_2, rel1), \ppreceq(c_2,c_1, rel2), rel1 \neq rel2.\\[0.5ex]
\mbox{(props-forall)} 
   & \instd(y,z',c,t) & \rif & 
	 \supForall(z,r,z',c_1), \instd(x,z,c,t), \tripled(x,r,y,c,t),\\
   &&& \pprec(c,c_2, rel1), \ppreceq(c_2,c_1, rel2), rel1 \neq rel2.\\[0.5ex]
\mbox{(props-leqone)} & \unsat(t) & \rif & 
  \supLeqOne(z,r,c_1), \instd(x,z,c,t),\\													
   &&& \tripled(x,r,x_1,c,t), \tripled(x,r,x_2,c,t),\\
   &&& \pprec(c,c_2, rel1), \ppreceq(c_2,c_1, rel2), rel1 \neq rel2.\\[0.5ex]
   \mbox{(props-subr)} 
   & \tripled(x,w,x',c,t) & \rif & 
	 \subRole(v,w,c_1), \tripled(x,v,x',c,t), \\
   &&& \pprec(c,c_2, rel1), \ppreceq(c_2,c_1, rel2), rel1 \neq rel2.\\[0.5ex]
\mbox{(props-subrc)} 
   & \tripled(x,w,z,c,t) & \rif & 
	\subRChain(u,v,w,c_1), \tripled(x,u,y,c,t), \tripled(y,v,z,c,t),\\
	&&& \pprec(c,c_2, rel1), \ppreceq(c_2,c_1, rel2), rel1 \neq rel2.\\[0.5ex]
\mbox{(props-dis)} 
   & \unsat(t) & \rif & 
	\pDis(u,v,c_1), \tripled(x,u,y,c,t), \tripled(x,v,y,c,t),\\
	&&& \pprec(c,c_2, rel1), \ppreceq(c_2,c_1, rel2), rel1 \neq rel2.\\[0.5ex]
   \mbox{(props-inv1)} 
   & \tripled(y,v,x,c,t) & \rif & 
	\pInv(u,v,c_1), \tripled(x,u,y,c,t), \\
	&&& \pprec(c,c_2, rel1), \ppreceq(c_2,c_1, rel2), rel1 \neq rel2.\\
   \mbox{(props-inv2)} 
   & \tripled(x,u,y,c,t) & \rif & 
	\pInv(u,v,c_1), \tripled(y,v,x,c,t), \\
	&&& \pprec(c,c_2, rel1), \ppreceq(c_2,c_1, rel2), rel1 \neq rel2.\\[0.5ex]
 \mbox{(props-irr)} 
   & \unsat(t) & \rif & 
	\pIrr(u,c_1),\tripled(x,u,x,c,t),\\
	&&& \pprec(c,c_2, rel1), \ppreceq(c_2,c_1, rel2), rel1 \neq rel2.\\[1.5ex]
\end{array}$}
\hrule
\end{table}

\begin{table}[tp]%

\bigskip

\caption{Deduction rules $P_{\default}$ for defeasible axioms: defeasible inheritance rules}
\label{tab:def-inheritance-rules-tgl}

\medskip

\hrule\mbox{}\\[1ex]
\scalebox{.9}{
\small
$\begin{array}{l@{\;}r@{\ }r@{\ }l@{}}
   \mbox{(propd-subc)} 
   &    \instd(x,z,c,t) & \rif & \defsubs(y,z,c_1, rel1), \instd(x,y,c,t),\\
    &&& \pprec(c,c_2, rel1), \ppreceq(c_2,c_1, rel2), rel1 \neq rel2,\\
		&&& \naf \ovr(\subClass,x,y,z,c_1,c, rel1).\\[0.5ex]
   \mbox{(propd-cnj)} 
   & \instd(x,z,c,t) & \rif & 
	 \defsubcnj(y_1,y_2,z,c_1, rel1), \instd(x,y_1,c,t),\instd(x,y_2,c,t),\\
	 &&& \pprec(c,c_2, rel1), \ppreceq(c_2,c_1, rel2), rel1 \neq rel2,\\
   &&& \naf \ovr(\subConj,x,y_1,y_2,z,c_1,c, rel1). \\[0.5ex]
  \mbox{(propd-subex)} 
   & \instd(x,z,c,t) & \rif & 
	\defsubex(v,y,z,c_1, rel1), \tripled(x,v,x',c,t),\instd(x',y,c,t),\\
   &&& \pprec(c,c_2, rel1), \ppreceq(c_2,c_1, rel2), rel1 \neq rel2,\\
	 &&& \naf \ovr(\subEx,x,v,y,z,c_1,c, rel1). \\[0.5ex]
   \mbox{(propd-supex)} 
   & \tripled(x,r,x',c,t) & \rif & 
	 \defsupex(y,r,x',c_1, rel1), \instd(x,y,c,t),\\
	 &&& \pprec(c,c_2, rel1), \ppreceq(c_2,c_1, rel2), rel1 \neq rel2,\\
   &&& \naf \ovr(\supEx,x,y,r,x',c_1,c, rel1). \\[0.5ex]
\mbox{(propd-forall)} 
   & \instd(y,z',c,t) & \rif & 
	 \defsupforall(z,r,z',c_1, rel1), \instd(x,z,c,t), \tripled(x,r,y,c,t),\\
	&&& \pprec(c,c_2, rel1), \ppreceq(c_2,c_1, rel2), rel1 \neq rel2,\\
  &&& \naf \ovr(\supForall,x,y,z,r,z',c_1,c, rel1).\\[0.5ex]      
\mbox{(propd-leqone)} & \unsat(t) & \rif & 
 \defsupleqone(z,r,c_1, rel1), \instd(x,z,c,t),\\													
  &&& \tripled(x,r,x_1,c,t), \tripled(x,r,x_2,c,t),\\
  &&& \pprec(c,c_2, rel1), \ppreceq(c_2,c_1, rel2), rel1 \neq rel2,\\
	&&& \naf \ovr(\supLeqOne,x,x_1,x_2,z,r,c_1,c, rel1).\\[0.5ex]
  \mbox{(propd-subr)} 
  & \tripled(x,w,x',c,t) & \rif & 
	\defsubr(v,w,c_1, rel1), \tripled(x,v,x',c,t), \\
  &&& \pprec(c,c_2, rel1), \ppreceq(c_2,c_1, rel2), rel1 \neq rel2,\\
	&&& \naf \ovr(\subRole,x,y,v,w,c_1,c, rel1). \\[0.5ex]
\mbox{(propd-subrc)} 
   & \tripled(x,w,z,c,t) & \rif & 
	\defsubrc(u,v,w,c_1, rel1), \tripled(x,u,y,c,t), \tripled(y,v,z,c,t),\\
	&&& \pprec(c,c_2, rel1), \ppreceq(c_2,c_1, rel2), rel1 \neq rel2,\\
  &&& \naf \ovr(\subRChain,x,y,z,u,v,w,c_1,c, rel1). \\[0.5ex]
\mbox{(propd-dis)} 
   & \unsat(t) & \rif & 
	\defdis(u,v,c_1, rel1), \tripled(x,u,y,c,t), \tripled(x,v,y,c,t),\\
	&&& \pprec(c,c_2, rel1), \ppreceq(c_2,c_1, rel2), rel1 \neq rel2,\\
  &&& \naf \ovr(\pDis,x,y,u,v,c_1,c, rel1). \\[0.5ex]             
   \mbox{(propd-inv1)} 
   & \tripled(y,v,x,c,t) & \rif & 
	 \definv(u,v,c_1, rel1), \tripled(x,u,y,c,t), \\
	&&& \pprec(c,c_2, rel1), \ppreceq(c_2,c_1, rel2), rel1 \neq rel2,\\
  &&& \naf \ovr(\pInv,x,y,u,v,c_1,c, rel1). \\
   \mbox{(propd-inv2)} 
   & \tripled(x,u,y,c,t) & \rif & 
	 \definv(u,v,c_1, rel1), \tripled(y,v,x,c,t), \\
	&&& \pprec(c,c_2, rel1), \ppreceq(c_2,c_1, rel2), rel1 \neq rel2,\\
  &&& \naf \ovr(\pInv,x,y,u,v,c_1,c, rel1). \\[0.5ex]              
 \mbox{(propd-irr)} 
   & \unsat(t) & \rif & 
	\defirr(u,c_1, rel1), \tripled(x,u,x,c,t),\\
	&&& \pprec(c,c_2, rel1), \ppreceq(c_2,c_1, rel2), rel1 \neq rel2,\\
  &&& \naf \ovr(\pIrr,x,u,c_1,c, rel1). \\[1.5ex]
\end{array}$}
\hrule
\end{table}

\begin{table}[tp]%

\bigskip

\caption{Deduction rules $P_{\default}$ for defeasible axioms: parallel inheritance rules}
\label{tab:par-inheritance-rules-tgl}

\medskip

\hrule\mbox{}\\[1ex]
\scalebox{.9}{
\small
$\begin{array}{l@{\;}r@{\ }r@{\ }l@{}}
	
   \mbox{(propp-subc)} 
   &    \instd(x,z,c,t) & \rif & \defsubs(y,z,c_1, rel1), \instd(x,y,c,t),\\
    &&& \ppreceq(c ,c_1, rel2), rel1 \neq rel2.\\[0.5ex]
		
   \mbox{(propp-cnj)} 
   & \instd(x,z,c,t) & \rif & 
	 \defsubcnj(y_1,y_2,z,c_1, rel1), \instd(x,y_1,c,t),\instd(x,y_2,c,t),\\
	 &&& \ppreceq(c,c_1, rel2), rel1 \neq rel2.\\[0.5ex]
  \mbox{(propp-subex)} 
   & \instd(x,z,c,t) & \rif & 
	\defsubex(v,y,z,c_1, rel1), \tripled(x,v,x',c,t),\instd(x',y,c,t),\\
   &&& \ppreceq(c,c_1, rel2), rel1 \neq rel2.\\[0.5ex]
   \mbox{(propp-supex)} 
   & \tripled(x,r,x',c,t) & \rif & 
	 \defsupex(y,r,x',c_1, rel1), \instd(x,y,c,t),\\
	 &&& \ppreceq(c,c_1, rel2), rel1 \neq rel2.\\[0.5ex]
\mbox{(propp-forall)} 
   & \instd(y,z',c,t) & \rif & 
	 \defsupforall(z,r,z',c_1, rel1), \instd(x,z,c,t), \tripled(x,r,y,c,t),\\
	&&& \ppreceq(c,c_1, rel2), rel1 \neq rel2.\\[0.5ex]      
\mbox{(propp-leqone)} & \unsat(t) & \rif & 
 \defsupleqone(z,r,c_1, rel1), \instd(x,z,c,t),\\													
  &&& \tripled(x,r,x_1,c,t), \tripled(x,r,x_2,c,t),\\
  &&& \ppreceq(c,c_1, rel2), rel1 \neq rel2.\\[0.5ex]

  \mbox{(propp-subr)} 
  & \tripled(x,w,x',c,t) & \rif & 
	\defsubr(v,w,c_1, rel1), \tripled(x,v,x',c,t), \\
  &&& \ppreceq(c,c_1, rel2), rel1 \neq rel2.\\[0.5ex]
\mbox{(propp-subrc)} 
   & \tripled(x,w,z,c,t) & \rif & 
	\defsubrc(u,v,w,c_1, rel1), \tripled(x,u,y,c,t), \tripled(y,v,z,c,t),\\
	&&& \ppreceq(c,c_1, rel2), rel1 \neq rel2.\\[0.5ex]
\mbox{(propp-dis)} 
   & \unsat(t) & \rif & 
	\defdis(u,v,c_1, rel1), \tripled(x,u,y,c,t), \tripled(x,v,y,c,t),\\
	&&& \ppreceq(c,c_1, rel2), rel1 \neq rel2.\\[0.5ex]             
   \mbox{(propp-inv1)} 
   & \tripled(y,v,x,c,t) & \rif & 
	 \definv(u,v,c_1, rel1), \tripled(x,u,y,c,t), \\
	&&& \ppreceq(c,c_1, rel2), rel1 \neq rel2.\\
   \mbox{(propp-inv2)} 
   & \tripled(x,u,y,c,t) & \rif & 
	 \definv(u,v,c_1, rel1), \tripled(y,v,x,c,t), \\
	&&& \ppreceq(c,c_1, rel2), rel1 \neq rel2.\\[0.5ex]              
 \mbox{(propp-irr)} 
   & \unsat(t) & \rif & 
	\defirr(u,c_1, rel1), \tripled(x,u,x,c,t),\\
	&&& \ppreceq(c,c_1, rel2), rel1 \neq rel2.\\[1.5ex]
\end{array}$}
\hrule
\end{table}

\begin{table}[tp]%
\caption{Deduction rules $P_{\default}$ for defeasible axioms: test rules}
\label{tab:test-rules-tgl} 

\bigskip

\hrule\mbox{}\\[1ex]
\scalebox{.85}{
\small
$\begin{array}{l@{\!\!\!\!}r@{\ }r@{\ }l@{}}

	
   \mbox{(test-subc)} 
   &    \test(\nlit(x,z,c)) 
	 & \rif & \defsubs(y,z,c_1, rel1), \instd(x,y,c,\ml{main}),\\ 
	 &&& \pprec(c,c_2, rel1), \ppreceq(c_2,c_1, rel2), rel1 \neq rel2.\\[0.5ex]
   \mbox{(constr-subc)} 
   &  & \rif & \testf(\nlit(x,z,c)), \ovr(\subClass,x,y,z,c_1,c, rel).\\[1ex]
	
   \mbox{(test-subcnj)} 
   &    \test(\nlit(x,z,c)) & \rif & 
	\defsubcnj(y_1,y_2,z,c_1, rel1),\\ 
  &&& \instd(x,y_1,c,\ml{main}), \instd(x,y_2,c,\ml{main}),\\ 
	&&& \pprec(c,c_2, rel1), \ppreceq(c_2,c_1, rel2), rel1 \neq rel2.\\[0.5ex]
   \mbox{(constr-subcnj)} 
   &  & \rif & \testf(\nlit(x,z,c)), \ovr(\subConj,x,y_1,y_2,z,c_1,c, rel).\\[1ex]
   \mbox{(test-subex)} 
   &    \test(\nlit(x,z,c)) & \rif & 
	 \defsubex(r,y,z,c_1, rel1), \\
   &&& \tripled(x,r,w,c,\ml{main}), \instd(w,y,c,\ml{main}),\\
	 &&& \pprec(c,c_2, rel1), \ppreceq(c_2,c_1, rel2), rel1 \neq rel2.\\[0.5ex]
	 \mbox{(constr-subex)} 
   &  & \rif & \testf(\nlit(x,z,c)), \ovr(\subEx,x,r,y,z,c_1,c, rel).\\[1ex]
   \mbox{(test-supex)} 
   &  \test(\nrel(x,r,w,c)) & \rif & 
	 \defsupex(y,r,w,c_1, rel1), \instd(x,y,c,\ml{main}),\\ 
	 &&& \pprec(c,c_2, rel1), \ppreceq(c_2,c_1, rel2), rel1 \neq rel2.\\[0.5ex]
   \mbox{(constr-supex)} 
   &  & \rif & \testf(\nrel(x,r,w,c)), \ovr(\supEx,x,r,y,w,c_1,c, rel).\\[1ex]
   \mbox{(test-supforall)} 
   &    \test(\nlit(y,w,c)) & \rif & 
	 \defsupforall(z,r,w,c_1, rel1),\\ 
	 &&& \instd(x,z,c,\ml{main}), \tripled(x,r,y,c,\ml{main}),\\
   &&& \pprec(c,c_2, rel1), \ppreceq(c_2,c_1, rel2), rel1 \neq rel2.\\[0.5ex]
   \mbox{(constr-supforall)} 
   &  & \rif & \testf(\nlit(y,w,c)), \ovr(\supForall,x,y,z,r,w,c_1,c, rel).\\[1ex]   
   
	 \mbox{(test-subr)} 
   & \test(\nrel(x,s,y,c)) & \rif & 
	 \defsubr(r,s,c_1, rel1), \tripled(x,r,y,c,\ml{main}),\\ 
	 &&& \pprec(c,c_2, rel1), \ppreceq(c_2,c_1, rel2), rel1 \neq rel2.\\[0.5ex]
   \mbox{(constr-subr)} 
   &  & \rif & \testf(\nrel(x,s,y,c)), \ovr(\subRole,x,r,y,s,c_1,c, rel).\\[1ex]
   \mbox{(test-subrc)} 
   &    \test(\nrel(x,t,z,c)) & \rif & 
	 \defsubrc(r,s,t,c_1, rel1),\\ 
	&&& \tripled(x,r,y,c,\main), \tripled(y,s,z,c,\main),\\
	&&& \pprec(c,c_2, rel1), \ppreceq(c_2,c_1, rel2), rel1 \neq rel2.\\[0.5ex]
   \mbox{(constr-subrc)} 
   &  & \rif & \testf(\nrel(x,t,z,c)), \ovr(\subRChain,x,y,z,r,s,t,c_1,c, rel).\\[1ex]
	
   \mbox{(test-inv1)} 
   &    \test(\nrel(x,s,y,c)) & \rif & 
	  \definv(r,s,c_1, rel1), \tripled(x,r,y,c,\main),\\
	 &&& \pprec(c,c_2, rel1), \ppreceq(c_2,c_1, rel2), rel1 \neq rel2.\\
   \mbox{(test-inv2)} 
   &    \test(\nrel(y,r,x,c)) & \rif & 
	 \definv(r,s,c_1, rel), \tripled(x,s,y,c,\main),\\ 
	 &&& \pprec(c,c_2, rel1), \ppreceq(c_2,c_1, rel2), rel1 \neq rel2.\\[0.5ex]
   \mbox{(constr-inv1)} 
   &  & \rif & \naf \testf(\nrel(x,s,y,c)), \ovr(\pInv,x,y,r,s,c_1,c, rel).\\
   \mbox{(constr-inv2)} 
   &  & \rif & \naf \testf(\nrel(y,r,x,c)), \ovr(\pInv,x,y,r,s,c_1,c, rel).\\[2ex]
	
   \mbox{(test-fails1)} 
   &  \testf(\nlit(x,z,c)) & \rif & \instd(x,z,c,\nlit(x,z,c)), \naf \unsat(\nlit(x,z,c)).\\
   \mbox{(test-fails2)} 
   &  \testf(\nrel(x,r,y,c)) & \rif & \tripled(x,r,y,c,\nrel(x,r,y,c)), \naf \unsat(\nrel(x,r,y,c)).\\[1ex]

   \mbox{(test-add1)} 
   &  \instd(x, z, c, \nlit(x,z,c)) & \rif & \test(\nlit(x,z,c)).\\
   \mbox{(test-add2)} 
   &  \tripled(x, r, y, c, \nrel(x,r,y,c)) & \rif & \test(\nrel(x,r,y,c)).\\[1ex]

   \mbox{(test-copy1)} 
   &  \instd(x_1, y_1, c, t) & \rif & \instd(x_1,y_1,c,\ml{main}), \test(t).\\
   \mbox{(test-copy2)} 
   &  \tripled(x_1, r, y_1, c, t) & \rif & \tripled(x_1,r,y_1,c,\ml{main}), \test(t).\\[1ex]
\end{array}$}
\hrule
\end{table}

\begin{table}[tp]%
\caption{Rules in $P_{pref}$ for preference definitions: preparation rules}
\label{tab:prep-rules-tgl} 

\bigskip

\hrule\mbox{}\\[1ex]
\scalebox{.9}{
\small
$\begin{array}{l@{\ }r@{\ }r@{\ }l@{}}

   \mbox{(prep-indiv)} 
   &    \indivi(x) & \rif & \nom(x,c).\\[.5ex]
   \mbox{(prep-ovr-subs)} 
   &    \possovr(\subClass(x,y,z),c,rel) & \rif & \defsubs(y,z,c,rel), \\
   & & &  \indivi(x).\\[.5ex]
   \mbox{(prep-ovr-subc)} 
   &    \possovr(\subConj(x,y1,y2,z),c,rel) & \rif & \defsubcnj(y1,y2,z,c,rel), \\
   & & &  \indivi(x).\\[.5ex]
   \mbox{(prep-ovr-subex)} 
   &    \possovr(\subEx(x,r,y,z),c,rel) & \rif & \defsubex(r,y,z,c,rel), \\
   & & &  \indivi(x).\\[.5ex]
   \mbox{(prep-ovr-supex)} 
   &    \possovr(\supEx(x,y,r,w),c,rel)  & \rif & \defsupex(y,r,w,c,rel), \\
   & & &  \indivi(x).\\[.5ex]
   \mbox{(prep-ovr-supfa)} 
   &    \possovr(\supForall(x,y,z,r,w),c,rel) & \rif & \defsupforall(z,r,w,c,rel), \\
   & & & \indivi(x), \indivi(y).\\[.5ex]
   \mbox{(prep-ovr-suble)} 
   &    \possovr(\supLeqOne(x,x1,x2,z,r),c,rel) & \rif & \defsupleqone(z,r,c,rel), \\
   & & &  \indivi(x), \indivi(x1), \indivi(x2).\\[.5ex]
   \mbox{(prep-ovr-subr)} 
   &    \possovr(\subRole(x,y,r,s),c,rel) & \rif & \defsubr(r,s,c,rel), \\
   & & &  \indivi(x), \indivi(y).\\[.5ex]
   \mbox{(prep-ovr-subrc)} 
   &    \possovr(\subRChain(x,y,z,r,s,t),c,rel) & \rif & \defsubrc(r,s,t,c,rel), \\
   & & &  \indivi(x), \indivi(y), \indivi(z).\\[.5ex]
   \mbox{(prep-ovr-dis)} 
   &    \possovr(\pDis(x,y,r,s),c,rel)  & \rif & \defdis(r,s,c,rel), \\
   & & &  \indivi(x), \indivi(y).\\[.5ex]
   \mbox{(prep-ovr-inv)} 
   &    \possovr(\pInv(x,y,r,s),c,rel)  & \rif & \definv(r,s,c,rel), \\
   & & &  \indivi(x), \indivi(y).\\[.5ex]
   \mbox{(prep-ovr-irr)} 
   &    \possovr(\pIrr(x,r),c,rel) & \rif & \defirr(r,c,rel), \\
   & & &  \indivi(x).\\[2ex]
   
   \mbox{(act-ovr-subs)} 
   &    \ovr(\subClass(x,y,z),c_1,c,rel) & \rif & \ovr(\subClass,x,y,z,c_1,c,rel)\\[.5ex]
   \mbox{(act-ovr-subc)} 
   &    \ovr(\subConj(x,y1,y2,z),c_1,c,rel) & \rif & \ovr(\subConj,x,y1,y2,z,c_1,c,rel).\\[.5ex]
   \mbox{(act-ovr-subex)} 
   &    \ovr(\subEx(x,r,y,z),c_1,c,rel) & \rif & \ovr(\subEx,x,r,y,z,c_1,c,rel).\\[.5ex]
   \mbox{(act-ovr-supex)} 
   &    \ovr(\supEx(x,y,r,w),c_1,c,rel)  & \rif & \ovr(\supEx,x,y,r,w,c_1,c,rel).\\[.5ex]
   \mbox{(act-ovr-supfa)} 
   &    \ovr(\supForall(x,y,z,r,w),c_1,c,rel) & \rif & \ovr(\supForall,x,y,z,r,w,c_1,c,rel).\\[.5ex]
   \mbox{(act-ovr-suble)} 
   &    \ovr(\supLeqOne(x,x1,x2,z,r),c_1,c,rel) & \rif & \ovr(\supLeqOne,x,x1,x2,z,r,c_1,c,rel).\\[.5ex]
   \mbox{(act-ovr-subr)} 
   &    \ovr(\subRole(x,y,r,s),c_1,c,rel) & \rif & \ovr(\subRole,x,y,r,s,c_1,c,rel).\\[.5ex]
   \mbox{(act-ovr-subrc)} 
   &    \ovr(\subRChain(x,y,z,r,s,t),c_1,c,rel) & \rif & \ovr(\subRChain,x,y,z,r,s,t,c_1,c,rel).\\[.5ex]
   \mbox{(act-ovr-dis)} 
   &    \ovr(\pDis(x,y,r,s),c_1,c,rel)  & \rif & \ovr(\pDis,x,y,r,s,c_1,c,rel).\\[.5ex]
   \mbox{(act-ovr-inv)} 
   &    \ovr(\pInv(x,y,r,s),c_1,c,rel)  & \rif & \ovr(\pInv,x,y,r,s,c_1,c,rel).\\[.5ex]
   \mbox{(act-ovr-irr)} 
   &    \ovr(\pIrr(x,r),c_1,c,rel) & \rif & \ovr(\pIrr,x,r,c_1,c,rel).\\[1ex]
\end{array}$}
\hrule
\end{table}

\begin{table}[tp]%
\caption{asprin program $P_{pref}$ for preference definitions}
\label{tab:preference-tbl} 

\bigskip

\hrule\mbox{}\\[1ex]
\begin{tabular}{ll}

\medskip
\mbox{(pref-local)} 
   &  \begin{lstlisting}[mathescape=true]
#preference(LocPref(C, REL), poset){ 
    $\neg \ovr($A,Cp, C,REL)  >> $\ovr($A,Cp,C,REL);
    $\neg \ovr($A1,C1,C,REL) >> $\neg \ovr($A2,C2,C,REL) : 
        $\preceqex$(C1b,C1,REL), $\preceqex$(C2b,C2,REL), 
        $\pprec$(C,C2b,REL), $\pprec$(C2b,C1b,REL), 
        $\possovr$(A1,C1,REL), $\possovr$(A2,C2,REL).
} : context(C), relation(REL).
\end{lstlisting}\\

\medskip
\mbox{(pref-rel\_local)} 
   &  \begin{lstlisting}
#preference(RelPref(REL), pareto){ 
	**LocPref(C, REL) : context(C)
} : relation(REL).
\end{lstlisting}\\

\medskip
\mbox{(pref-global)} 
   &  \begin{lstlisting}
#preference(GlobPref, lexico){ 
    W::**RelPref(REL) : relation_weight(REL, W)
}.
\end{lstlisting} \\

\medskip
\mbox{(pref-optima)} 
   &  \begin{lstlisting}
#optimize(GlobPref).
\end{lstlisting}
\end{tabular}
\\ 

\medskip

\hrule
\end{table}

\newpage

\section{Translation correctness: more details}

Let us consider a CAS-interpretation 
$\IC_{\CAS} = \stru{\IC, \ov{\casmap}}$ 
with $\ov{\casmap} = (\casmap_t, \casmap_c)$.
%
We construct its set of atoms corresponding to its overriding assumptions as:
\begin{center}
$\OVR(\IC_\CAS) = 
\{\, \ovr(p(\ee),\mlc,rel) \;|\; \stru{\alpha, \ee} \in \casmap_{rel}(\mlc),\, 
                                 I_{rl}(\alpha, \mlc_1) = p \,\}.$
\end{center}
We can build\footnote{See similar construction in~\cite[Section~A.5.2]{BozzatoES:18} for 
further details.} from its components a corresponding Herbrand interpretation 
$I(\IC_{\CAS})$ of the program $PK(\CKB)$ as the smallest set of literals containing:

\begin{itemize}
\itemsep=0pt
\item 
all facts of $PK(\CKB)$; 
\item 
 $\instd(a, A, \mlc, \smlmain)$, if $\I(\mlc) \models A(a)$; 
\item 
  $\tripled(a,R,b, \mlc, \smlmain)$, if  $\I(\mlc) \models R(a,b)$;
\item
  each $\ovr$-literal from $\OVR(\IC_{\CAS})$;
\item 
  each literal $l$ with environment $t \neq \smlmain$,
  if $\test(t) \in I(\IC_{\CAS})$ and 
	$l$ is in the head of a rule $r \in \grd(PK(\CKB))$
	with $\Body(r) \subseteq I(\IC_{\CAS})$;
\item
  $\test(t)$, if
  $\testf(t)$ appears in the body of an overriding rule $r$
  in $\grd(PK(\CKB))$ and the head of $r$ is an $\ovr$ literal in
  $\OVR(\IC_{\CAS})$;  
\item
  $\unsat(t) \in I(\IC_{\CAS})$, if
  adding the literal corresponding to $t$ to the local interpretation of its
  context $\mlc$ violates some axiom of the local knowledge $\KB_\mlc$;
\item
  $\testf(t)$, if $\unsat(t) \notin I(\IC_{\CAS})$.
\end{itemize}
Note that $\unsat(\smlmain)$ is not included in $I(\IC_{\CAS})$.
Moreover, as all the facts of $PK(\CKB)$ are included in
the set, also the atoms $\pprec$ and $\ppreceq$
defining the contextual relations of $\CKB$ are included in 
$I(\IC_{\CAS})$.

The correctness result provided by Theorem~\ref{thm:encode} in Section~\ref{sec:asprin_enc}
is a consequence of the following Lemma~\ref{lem:local-correctness}, showing 
the correspondence between the minimal justified CKR-models of $\CKB$
and the answer sets of $PK(\CKB)$, and
Lemma~\ref{lem:pref-correctness}, proving the correspondence
between preferred models and answer sets selected by the asprin preference in $P_{pref}$.

\begin{lemma} 
\label{lem:local-correctness}
Let $\CKB$ be a multi-relational sCKR in $\SROIQrld$ normal form, then:
	\begin{enumerate}[label=(\roman*).]
	\item 
	  for every (named) justified clashing assumption $\ov{\casmap}$, 
		the interpretation $S = I(\hat{\IC}(\ov{\casmap}))$ is an answer set of $PK(\CKB)$;
	\item
	   every answer set $S$ of $PK(\CKB)$ is of the form 
		 $S = I(\hat{\IC}(\ov{\casmap}))$ with $\ov{\casmap}$ 
		 a (named) justified clashing assumption for $\CKB$.
	\end{enumerate}
\end{lemma}
\begin{proof}[Proof (Sketch)]
	Intuitively, we are interested in computing the correspondence with all
	(not necessarily preferred) answer sets of $PK(\CKB)$:
	we can show that the new form of rules for managing multiple 
	contextual relations do not influence the
	construction of such answer sets, thus 
	the result can be proved similarly to Lemma 6 in~\cite{BozzatoES:18}
	and its extension to hierarchies in~\cite[Lemma 1]{DBLP:conf/kr/BozzatoSE18}.
	
	Let us consider the interpretation $S = I(\hat{\IC}(\ov{\casmap}))$ 
	as defined above and the reduct $G_S(PK(\CKB))$ of $PK(\CKB)$ with respect $S$.
	The lemma can then be proved 
	by showing that the answer sets of $PK(\CKB)$ 
	coincide with the sets $S = I(\hat{\IC}(\ov{\casmap}))$ where $\ov{\casmap} = (\chi_t, \chi_c)$ 
	is composed by justified clashing assumptions of $\CKB$.
	
\smallskip
\noindent\textbf{(i).}\ \ \ Assuming that $\ov{\casmap} = (\chi_t, \chi_c)$ 
 is justified, we show that $S = I(\hat{\IC}(\ov{\casmap}))$ 
 is an answer set of $PK(\CKB)$.
 
 We first prove that $S \models G_{S}(PK(\CKB))$, that is for every
 rule instance $r \in G_{S}(PK(\CKB))$ it holds that $S \models r$.  
 This is proved by examining the possible rule forms that
 occur in $G_{S}(PK(\CKB))$.
 Here we show some representative cases (see also~\cite{BozzatoES:18}):
\begin{itemize}
\item
	 \textbf{(prl-instd):}
    then $\insta(a,A,\mlc) \in I(\hat{\IC}(\ov{\casmap}))$ and, by definition of the
		translation, $A(a) \in \KB_\mlc$. 
		This implies that $\I(\mlc) \models A(a)$
		and thus 
		$\instd(a,A,\mlc,\smlmain)$ is added to $I(\hat{\IC}(\ov{\casmap}))$.
	\item
	   \textbf{(prl-subc):}
    then $\{\subClass(A,B,\mlc), \instd(a,A,\mlc, t)\} \subseteq I(\hat{\IC}(\ov{\casmap}))$.
		By definition of the translation, we have $A \subs B \in \KB_\mlc$.
		For the construction of $I(\hat{\IC}(\ov{\casmap}))$, 
		if $t=\smlmain$ then $\I(\mlc) \models A(a)$. 
		This implies that $\I(\mlc) \models B(a)$ 
		and $\instd(a,B,\mlc,\smlmain)$ is added to $I(\hat{\IC}(\ov{\casmap}))$.
		Otherwise, if $t \neq \smlmain$ then
		$\instd(a,B,\mlc,t)$ is directly added to $I(\hat{\IC}(\casmap))$ by its construction.
	\item
	   \textbf{(ovr-subc):}
			 then $\{\defsubs(A,B,\mlc_1, rel1),$ $\pprec(\mlc,\mlc_2, rel1),$ 
			$\ppreceq(\mlc_2,\mlc_1, rel2),$\linebreak 
			$\instd(a,A,\mlc,\smlmain)\} \subseteq I(\hat{\IC}(\ov{\casmap}))$.
	   Since $r\in G_{S}(PK(\CKB))$,
	   then $\testf(\nlit(a,B,\mlc)) \notin I(\hat{\IC}(\ov{\casmap}))$.		
 	 By construction of $I(\hat{\IC}(\ov{\casmap}))$, this implies that
		 $\unsat(\nlit(a,B,\mlc)) \in I(\hat{\IC}(\ov{\casmap}))$,
		 meaning that $\I(\mlc) \models \non B(a)$.
		 Thus, $\I(\mlc)$ satisfies the clashing set $\{A(a), \non B(a)\}$
		 for the clashing assumption $\stru{A \subs B, a}$ for $rel1$ in context $\mlc$. 
		 Consequently, $\stru{A \subs B, a} \in \casmap_{rel1}(\mlc)$
		 and by construction $\ovr(\subClass,a,A,B,\mlc)$ is added to $I(\hat{\IC}(\ov{\casmap}))$.
	\item
	   \textbf{(props-subc):}
     then $\{\subClass(A,B,\mlc_1),$ $\instd(a,A,\mlc,t),$ 
		$\ppreceq(\mlc_2,\mlc_1, rel2),$\linebreak 
		$\pprec(\mlc,\mlc_2, rel1)\} \subseteq I(\hat{\IC}(\ov{\casmap}))$.
	   By definition, $A \subs B \in \KB_{\mlc_1}$ and, if $t = \smlmain$, $\I(\mlc) \models A(a)$.
	   Thus, for the definition of $\CAS$-model (condition (i) on strict axioms propagation),
		 $\instd(a,B,\mlc,t)$ is added to $I(\hat{\IC}(\ov{\casmap}))$.		
		 If $t \neq \smlmain$, then $\instd(a,B,\mlc,t)$ is added to $I(\hat{\IC}(\casmap))$ by construction.		
	\item
	   \textbf{(propd-subc):}
     then $\{\defsubs(A,B,\mlc_1, rel1), \instd(a,A,\mlc,t), 
		 \pprec(\mlc,\mlc_2, rel1),$\linebreak 
		$\ppreceq(\mlc_2,\mlc_1, rel2)\} \subseteq I(\hat{\IC}(\ov{\casmap}))$.
	   Since $r\in G_{S}(PK(\CKB))$,
		$\ovr(\subClass, a,A,B,\mlc_1,\mlc, rel1) \notin \OVR(\hat{\IC}(\ov{\casmap}))$ 
		 and hence $\stru{A \subs B, a} \notin \casmap_{rel1}(\mlc)$.
	   By definition, $D(A \subs B) \in \KB_{\mlc_1}$ and, if $t = \smlmain$, $\I(\mlc) \models A(a)$.
	   Thus, for the definition of $\CAS$-model (condition (iii) on defeasible axioms propagation),
		 $\instd(a,B,\mlc,t)$ is added to $I(\hat{\IC}(\ov{\casmap}))$.		
		 If $t \neq \smlmain$, then $\instd(a,B,\mlc,t)$ is added to $I(\hat{\IC}(\casmap))$ by construction.
	\item
	   \textbf{(propp-subc):}
     then $\{\defsubs(A,B,\mlc_1, rel1), \instd(a,A,\mlc,t),$  
		 $\ppreceq(\mlc,\mlc_1, rel2)\}$ $\subseteq I(\hat{\IC}(\ov{\casmap}))$.
	   By definition, $D(A \subs B) \in \KB_{\mlc_1}$ and, if $t = \smlmain$, $\I(\mlc) \models A(a)$
		 with $\mlc \prec_{rel2} \mlc_1$.
	   Thus, for the definition of $\CAS$-model (condition (ii) on propagation of defeasible
		 axioms over other relations),
		 $\instd(a,B,\mlc,t)$ is added to $I(\hat{\IC}(\ov{\casmap}))$.		
		 If $t \neq \smlmain$, then $\instd(a,B,\mlc,t)$ is added to $I(\hat{\IC}(\casmap))$ by construction.
	\item
	\textbf{(test-subc):}
		 then $\{\defsubs(A,B, \mlc_1, rel1), \instd(a,A,\mlc,\smlmain),$
		$\pprec(\mlc,\mlc_2, rel1),$\linebreak 
		$\ppreceq(\mlc_2,\mlc_1, rel2)\} \subseteq I(\hat{\IC}(\casmap))$.				 
		Thus $\default(A \subs B) \in \KB_{\mlc_1}$ and $\I(\mlc) \models A(a)$
		with $\mlc \prec_{rel1} \mlc_2 \preceq_{rel2} \mlc_1$.
		By the construction of $I(\hat{\IC}(\ov{\casmap}))$
		we have that $\test(\nlit(a,B,\mlc)) \in I(\hat{\IC}(\casmap))$.
\end{itemize}
Minimality of $S=I(\hat{\IC}(\casmap))$ w.r.t.\ the positive
  deduction rules of $G_{S}(PK(\CKB))$ can then be motivated
	as in the original proof in~\cite{BozzatoES:18}:
	thus, $I(\hat{\IC}(\ov{\casmap}))$ is an answer set of $PK(\CKB)$.

\smallskip
\noindent\textbf{(ii).}\ \ \ 
Let $S$ be an answer set of $PK(\CKB)$.
We show that there is some  justified clashing assumption
$\ov{\casmap}$ for $\CKB$ such that $S = I(\hat{\IC}(\ov{\casmap}))$ holds.

Note that as $S$ is an answer set for the CKR program,
all literals on $\ovr$ and $\testf$ in $S$ are derivable from the
reduct $G_{S}(PK(\CKB))$. 
By the definition of $I(\hat{\IC}(\ov{\casmap}))$ we can easily build 
a model $\IC_S = \stru{\IC_S, \ov{\chi}^S}$ from the answer set $S$ as follows:
 for every $c \in \N$,
  we build the local interpretation $\I_S(c)= \stru{\Delta_c, \cdot^{\I(c)}}$ as follows:  	

  \begin{itemize}
  \itemsep=0pt
  \item 
    $\Delta_c = \{d \;|\; d \in \NI \}$;  
  \item
	  $a^{\I(c)} = a$, for every $a \in \NI$;
	\item
	  $A^{\I(c)} = \{d \in \Delta_c \mid S \models \instd(d, A, c,\smlmain) \}$, for every $A \in \NC$;
	\item
		$R^{\I(c)} = \{(d,d') \in \Delta_c \times \Delta_c \,|\, 
		S \models \tripled(d, R, d', c,\smlmain) \}$ for $R \in \NR$;
  \end{itemize}
	Finally, $\ov{\chi}^S = (\chi^S_t, \chi^S_c)$ where
	$\casmap^S_{rel}(c) = \{\stru{\alpha, \ee} \mid I_{rl}(\alpha, c') = p, \ovr(p(\ee),c,rel) \in S \}$.
  We have to show that 
  $\IC_{S}$ meets the definition of a least justifed CAS-model for 
	a multi-relational $\CKB$, that is:
  \begin{enumerate}[label=(\roman*)]
\itemsep=0pt
 \item
   for every $\alpha \in \KB_\mlc$ (strict axiom), 
   and $\mlc' \preceq_{*}\mlc$, $\I_S(\mlc') \models \alpha$;
 \item
   for every $\default_i(\alpha) \in \KB_\mlc$ 
	 and $\mlc' \preceq_{-i} \mlc$, $\I_S(\mlc') \models \alpha$;
  \item
    for every $\default_i(\alpha) \in \KB_\mlc$ 
    and $\mlc'' \prec_i \mlc' \preceq_{-i} \mlc$,    
		if $\stru{\alpha,\vc{d}} \notin \casmap_i (\mlc'')$,
    then $\I_S(\mlc'') \models \phi_\alpha(\vc{d})$.
\end{enumerate}
Note that, since we are considering multi-relational CKRs 
based only on two relations (time and coverage), 
the relational closure $\mlc' \preceq_{-i} \mlc$ can be read 
as $\mlc' \preceq_{j} \mlc$ with $j \neq i$: 
this corresponds to the conditions 
$\ppreceq(\mlc', \mlc, rel2)$ with $rel1 \neq rel2$
in the formulation of the rules.

Item (i) should be proved in the local case where $\mlc'=\mlc$ and
in the ``strict propagation'' case where $\mlc'\prec_{*} \mlc$. The second case can be shown
similarly to the local case, considering strict propagation rules in 
Table~\ref{tab:strict-inheritance-rules-tgl}.
Thus, considering $\mlc'=\mlc$, we verify the condition 
by showing that, for every $\KB_\mlc$, 
we have $\I_S(\mlc) \models \KB_\mlc$.
This can be shown by cases considering the form of
all of the (strict) axioms $\beta \in \Lcal_{\Sigma,\N}$ that can occur in $\KB_\mlc$.
For example (the other cases are similar):
	\begin{itemize}
	\item 
	  Let $\beta = A(a) \in \KB_\mlc$, then, by rule (prl-instd),
	  $S \models \instd(a,A,\mlc,\smlmain)$. 
	  This directly implies that $a^{\I(c)} \in A^{\I(c)}$.		
	\item 
	  Let $\beta = A \subs B \in \KB_\mlc$, then
	  $S \models \subClass(A,B,\mlc)$. 
		If $d \in A^{\I(\mlc)}$,
	  then by definition\linebreak 
		$S \models \instd(d,A,\mlc,\smlmain)$.
	  By rule (prl-subc) we obtain that $S \models \instd(d,B,\mlc,\smlmain)$
	  and thus $d \in B^{\I(\mlc)}$.		
  \end{itemize}
	Condition (ii) can be proved similarly, considering rules of Table~\ref{tab:par-inheritance-rules-tgl}.
	In particular, assuming 
	that $\default_i(\beta) \in \KB_{\mlc'}$ with $\mlc \preceq_{-i} \mlc'$
	we can proceed by cases on the possible forms of $\beta$
	and consider the (strict) propagation of defeasible axioms to $\mlc$
	along the ``parallel'' relations. For example:
	\begin{itemize}
	\item 
	  Let $\beta = A(a)$. Then, by definition of the translation,
	  we have 
		$S \models \definst(a, A, \mlc',rel1)$.
		Moreover, since $\mlc \preceq_{rel2} \mlc'$,
		we have $S \models \ppreceq(\mlc, \mlc', rel2)$ with $rel1 \neq rel2$.
		By the corresponding instantiation of
		rule (propp-inst), it holds that
		$S \models \instd(a, A, \mlc,\smlmain)$.		
	  By definition, this means that $a^{\I(\mlc)} \in A^{\I(\mlc)}$.
	\item 
	  Let $\beta = A \subs B$. Then, by definition of the translation,
	  $S \models \defsubs(A, B, \mlc', rel1)$.
		Since $\mlc \preceq_{rel2} \mlc'$,
		we have $S \models \ppreceq(\mlc, \mlc', rel2)$ with $rel1 \neq rel2$.
		If $a^{\I(\mlc)} \in A^{\I(\mlc)}$,
	  then by definition $S \models \instd(a,A,\mlc,\smlmain)$:
	  by rule (propp-subc), we obtain that $S \models \instd(a,B,\mlc,\smlmain)$
	  and thus $a^{\I(\mlc)} \in B^{\I(\mlc)}$.
  \end{itemize}	
  To prove condition (iii), let us assume that $\default_i(\beta) \in \KB_{\mlc'}$ with 
	$\mlc \prec_i \mlc'' \preceq_{-i} \mlc'$.
  We proceed again by cases on the possible forms of $\beta$ as in the original proof in~\cite{BozzatoES:18},
	by considering the defeasible propagation to $\mlc$ along the relation $i$. For example:
	\begin{itemize}
	\item 
	  Let $\beta = A(a)$. Then, by definition of the translation,
	  we have that $S \models \definst(a, A, \mlc',rel1)$.
	  Suppose that 
		$\Pair{A(x)}{a} \notin \casmap^{S}_{rel1}(\mlc)$.
	  Then by definition, $\ovr(\insta, a, A, \mlc', \mlc, rel1) \notin \OVR(\hat{\IC}(\ov{\casmap}))$.
		By construction, we have $S \models \pprec(\mlc, \mlc'', rel1)$
		and $S \models \ppreceq(\mlc'', \mlc', rel2)$.		
	  By the definition of the reduction, the corresponding instantiation
	  of rule (propd-inst) has not been removed from $G_{S}(PK(\CKB))$: this implies that
	  $S \models \instd(a, A, \mlc,\smlmain)$.
	  By definition, this means that $a^{\I(\mlc)} \in A^{\I(\mlc)}$.
	\item 
	  Let $\beta = A \subs B$. Then, by definition of the translation,
	  $S \models \defsubs(A, B, \mlc', rel1)$.
		As above, we also have $S \models \pprec(\mlc, \mlc'', rel1)$
		and $S \models \ppreceq(\mlc'', \mlc', rel2)$.
	  Let us suppose that $b^{\I(\mlc)} \in A^{\I(\mlc)}$: then 
	  $S \models \instd(b, A, \mlc, \smlmain)$.
	  Suppose that $\Pair{A \subs B}{b} \notin \casmap_{S}(\mlc)$:
	  by definition,
		$\ovr(\subClass, b, A, B, \mlc', \mlc, rel1) \notin \OVR(\hat{\IC}(\ov{\casmap}))$.
	  By the definition of the reduction, the corresponding instantiation
	  of rule (propd-subc) has not been removed from $G_{S}(PK(\CKB))$: 
	  this implies that $S \models \instd(b, B, \mlc, \smlmain)$.
	  Thus, by definition, this means that $b^{\I(\mlc)} \in B^{\I(\mlc)}$.
 \end{itemize}
We have shown that $\IC_S$ is a CAS-model of $\CKB$:
using the same reasoning in the original proof in~\cite{BozzatoES:18}
we can also prove the $\IC_S$ corresponds to the least model 
and that $\ov{\chi}^S$ is justified, thus proving the result.
\end{proof}

\begin{lemma} 
\label{lem:pref-correctness}
  Let $\CKB$ be a multi-relational sCKR in $\SROIQrld$ normal form.
	Then, $\hat{\IC}$ is a CKR model of $\CKB$ iff
	there exists a (named) justified clashing assumption $\ov{\casmap}$ s.t.
	$I(\hat{\IC}(\ov{\chi}))$ is a preferred answer set of $PK(\CKB) \cup P_{pref}$.
\end{lemma}
For the proof we need the following result:

\begin{theorem}
\label{thm:pareto_equiv}
Let $\CKB$ be an \eval-disconnected sCKR and $\IC_{\CAS} = \stru{\IC, \casmap_1, \dots, \casmap_m}$ a justified model of $\CKB$. Then $\IC_{\CAS}$ is preferred with respect to $P_{1,i}$ defined by 
\begin{quote}
	$P_{1,i}(\stru{\IC^1, \casmap_1^1, \dots, \casmap_m^1}, \stru{\IC^2, \casmap_1^2, \dots, \casmap_m^2})$ iff	
  there exists some $\mlc \in \N$ s.t. 
	$\casmap_i^1(\mlc) > \casmap_i^2(\mlc)$ and not $\casmap_i^2(\mlc) > \casmap_i^1(\mlc)$, and
  for no context $\mlc' \neq \mlc \in \N$ it holds that 
	$\casmap_i^1(\mlc') < \casmap_i^2(\mlc')$ and not $\casmap_i^2(\mlc') < \casmap_i^1(\mlc')$.
\end{quote}
iff it is preferred with respect to $P_{2,i}$ defined by 
\begin{quote}
	$P_{2,i}(\stru{\IC^1, \casmap_1^1, \dots, \casmap_m^1}, \stru{\IC^2, \casmap_1^2, \dots, \casmap_m^2})$ iff	
  there exists some $\mlc \in \N$ s.t. 
	$\casmap_i^1(\mlc) > \casmap_i^2(\mlc)$ and not $\casmap_i^2(\mlc) > \casmap_i^1(\mlc)$, and
  for all contexts $\mlc' \in \N$ it holds that 
	$\casmap_i^1(\mlc') > \casmap_i^2(\mlc')$ or $\casmap_i^1(\mlc') = \casmap_i^2(\mlc')$.
\end{quote}
\end{theorem}
\begin{proof}[Proof (sketch) of Theorem~\ref{thm:pareto_equiv}]
$P_{2,i}(\stru{\IC^1, \casmap_1}, \stru{\IC^2, \casmap_2})$ implies $P_{1,i}(\stru{\IC^1, \casmap_1}, \stru{\IC^2, \casmap_2})$. So we consider the other direction.

Let $\IC_{\CAS}$ be preferred with respect to $P_{2,i}$. Assume that
        there exists a justified model $\IC_{\CAS}'$ of $\CKB$ such
        that $P_{1,i}(\IC_{\CAS}', \IC_{\CAS})$ holds. 

Let $\IC_{\CAS} = \stru{\{\mathcal{I}(\mlc)\}_{\mlc\in\N}, \casmap}$
        and $\IC_{\CAS}' = \stru{\{\mathcal{I}'(\mlc)\}_{\mlc\in\N},
        \casmap'}$. We know there exists some $\mlc^{*} \in \N$ such
        that $\casmap'(\mlc^{*}) > \casmap(\mlc^{*})$. This implies
        that some $\default(\alpha) \in \K_{\mlc'}$  and $\ee$ exist
        such that $\stru{\alpha, \ee} \in \casmap(\mlc^{*}) \setminus
        \casmap'(\mlc^{*})$. Let $C$ be the component of $DEP(\CKB)$
        that contains $X_{\mlc^{*}}$, where $X$ is any concept or role
        appearing in $\alpha$. Note that $C$ is independent of the choice of $X$, since any two possible choices $X, X'$ satisfy that $X_{\mlc^{*}}$ and $X'_{\mlc^{*}}$ are reachable from one another. 

We take $\IC_{\CAS}'' = \stru{\{\mathcal{I}''(\mlc)\}_{\mlc\in\N},
        \casmap''}$ such that $X^{\mathcal{I}''(\mlc)} =
        X^{\mathcal{I}(\mlc)}$ for $X_{\mlc} \not\in C$ and
        $X^{\mathcal{I}''(\mlc)} = X^{\mathcal{I}'(\mlc)}$ otherwise,
        and we let $\casmap''(\mlc) = \casmap(\mlc)$ for $\mlc \neq \mlc^{*}$ and $\casmap''(\mlc) = \casmap'(\mlc)$ otherwise. That is, we take the original justified model $\IC_{\CAS} $ and swap the interpretations of all the concepts and roles that were changed in order to satisfy $\alpha(\ee)$ at context $\mlc^{*}$ by their changed interpretation in $\IC_{\CAS}'$. The result, $\IC_{\CAS}''$, is still a model of $\CKB$, as we exchanged the interpretation for the whole component and therefore any relevant axioms stay satisfied, since they were satisfied in $\IC_{\CAS}'$. Furthermore, since $\CKB$ is \eval-disconnected, $\casmap''$ is justified because the default $\default(\alpha)$ does not use any concept/role $X$ such that $X_{\mlc^{*}}$ is connected to $X'_{\mlc}$ such that $\mlc \neq \mlc^{*}$ and $X$ is used in another default $\default(\beta)$. This implies that only the clashing assumptions for $\mlc^{*}$ were changed. 

Now, we however know that $P_{2,i}(\IC_{\CAS}'',\IC_{\CAS})$. This is
        a contradiction to our original assumption. Therefore, there
        cannot exist some $\IC_{\CAS}'$ such that $P_{1,i}(\IC_{\CAS}', \IC_{\CAS})$ and $\IC_{\CAS}$ is preferred with respect to $P_{1,i}$.
\end{proof}
\begin{proof}[Proof (sketch) of Lemma~\ref{lem:pref-correctness}]
Our definition of the preferences in $P_{pref}$ mirrors the definition of preference: 
both go from local preference on the clashing assumptions per context, i.e. $\casmap_i(\mlc)$, to per relation preference and finally to the global preference. We show that the definitions correspond for each step. 

We start with the local preference. So let $X, Y$ be two interpretations of $PK(\CKB)$, $\mlc$ a context and $i$ a relation. Then it holds that $X >_{\text{\lstinline!LocPref($\mlc$, $i$)!}} Y$ iff: 
\begin{itemize}
\item $X$ and $Y$ do not have the same clashing assumptions at $\mlc$ w.r.t.\ relation $i$;
\item for each $\neg \ovr(\alpha_1, e, \mlc_1, \mlc, i)$ s.t.\ $X \not\models \neg \ovr(\alpha_1, e, \mlc_1, \mlc, i)$ and $Y \models \neg \ovr(\alpha_1, e, \mlc_1, \mlc, i)$ there exists $\neg \ovr(\alpha_2, f, \mlc_2, \mlc, i) > \neg \ovr(\alpha_1, e, \mlc_1, \mlc, i)$ s.t.\ $X \models \neg \ovr(\alpha_2, f, \mlc_2, \mlc, i)$ and $Y \not \models \neg \ovr(\alpha_2, f, \mlc_2, \mlc, i)$.
\end{itemize}
or equivalently:
\begin{itemize}
\item $X$ and $Y$ do not have the same clashing assumptions at $\mlc$ w.r.t.\ relation $i$;
\item for each $\alpha_1, e$, where $\alpha_1$ is from context $\mlc_1 \succeq_{-i} \mlc_{1b} \succ_{i}  \mlc$, s.t.\ $X \models \ovr(\alpha_1, e, \mlc_1, \mlc, i)$ and $Y \not\models \ovr(\alpha_1, e, \mlc_1, \mlc, i)$ there exists $\alpha_2, f$, where $\alpha_2$ is from context $\mlc_2 \succeq_{-i} \mlc_{2b} \succ_i \mlc$ , s.t.\ $\mlc_{1b} \succ_i \mlc_{2b}$ and $X \not\models \ovr(\alpha_2, f, \mlc_2, \mlc, i)$ and $Y \models \ovr(\alpha_2, f, \mlc_2, \mlc, i)$.
\end{itemize}
The second item is equivalent to
\begin{quote}
  for every $\eta = \stru{\alpha_1,\ee} \in \chi_i^1(\mlc) \setminus \chi_i^2(\mlc)$ with 
  $\default_i(\alpha_1)$ at a context $\mlc_1 \succeq_{-i} \mlc_{1b} \succ_{i}  \mlc$, 
  there exists an $\eta' = \stru{\alpha_2,\ff} \in \chi_i^2(\mlc) \setminus \chi_i^1(\mlc)$ with 
	$\default_i(\alpha_2)$ at context $\mlc_2 \succeq_{-i} \mlc_{2b} \succ_i \mlc$ 
	such that $\mlc_{1b} \succ_i \mlc_{2b}$.
\end{quote}
So, we see that the only difference between $>_{\text{\lstinline!LocPref($\mlc$, $i$)!}}$ and the order on the context $\mlc$ is the first condition, i.e.\ that the clashing assumptions on $\mlc$ must be different. However, this does not affect us, since the definition of preference for justified interpretations always uses $E =$``$\casmap_i^1(\mlc) < \casmap_i^2(\mlc)$ and not $\casmap_i^2(\mlc) < \casmap_i^1(\mlc)$''. This is equivalent to ``$X >_{\text{\lstinline!LocPref($\mlc$, $i$)!}} Y$ and not $Y >_{\text{\lstinline!LocPref($\mlc$, $i$)!}} X$'', since $E$ can only hold when the clashing assumption sets at $\mlc$ w.r.t.\ relation $i$ are different.

Next, we consider the preference per relation. As we have shown in Theorem~\ref{thm:pareto_equiv} the preferred models with respect to the original preference relation $P_1$ are the same as the preferred models with respect to the preference relation $P_{2,i}$. However, as can be easily seen from the definition, $P_{2,i}$ is the order that has the models that are pareto optimal with respect to the local preference orders \lstinline*LocPref($\mlc$,$i$)* per context as its optimal models. We see that \lstinline*RelPref($i$)* correctly captures this, as it is the pareto combination of the orders \lstinline*LocPref($\mlc$,$i$)* for each context $\mlc$.

Last but not least, we consider the global preference. In our definition, we say that we prioritize the preference on the clashing assumptions with respect to the relations with a lower index. This corresponds to the lexicographical combination of the orders \lstinline*LocPref($i$)* for each relation $i$, when assigning the weight $i$ to relation $i$, when it is the preference relation with index $i$.
\end{proof}

\section{Proofs for Overall Weight Queries}
Before we define the semiring, we ensure that the preference relation \lstinline*LocPref($rel$)* is transitive.
\begin{lemma}
\label{lem:P2i-transitivity}
The preference relation \lstinline*LocPref($rel$)* defined in Section~\ref{sec:asprin_enc} is transitive.
\end{lemma}
We use the transitivity of the local preference:
\begin{lemma}
Let $\chi_i^1(\mlc) > \chi_i^2(\mlc)$ and $\chi_i^2(\mlc) > \chi_i^3(\mlc)$. Then $\chi_i^1(\mlc) > \chi_i^3(\mlc)$.
\end{lemma}
\begin{proof}
Assume $\chi_i^1(\mlc) > \chi_i^2(\mlc)$, $\chi_i^2(\mlc) > \chi_i^3(\mlc)$ and $\stru{\alpha_1,\ee} \in \chi_i^1(\mlc) \setminus \chi_i^3(\mlc)$ with $\default_i(\alpha_1)$ at a context $\mlc_1 \succeq_{-i} \mlc_{1b} \succ_{i}  \mlc$.

Case 1: If $\stru{\alpha_1,\ee} \not\in \chi_i^2(\mlc)$ then since $\chi_i^1(\mlc) > \chi_i^2(\mlc)$ there exists $\stru{\alpha_2,\ff} \in \chi_i^2(\mlc) \setminus \chi_i^1(\mlc)$ with $\default_i(\alpha_2)$ at context $\mlc_2 \succeq_{-i} \mlc_{2b} \succ_{i}  \mlc$ such that $\mlc_{1b} \succ_i \mlc_{2b}$.

Case 1.1: If $\stru{\alpha_2,\ff} \in \chi_i^3(\mlc)$ we are done. 

Case 1.2: Else, $\stru{\alpha_2,\ff} \in \chi_i^2(\mlc) \setminus \chi_i^3(\mlc)$. Then since $\chi_i^2(\mlc) > \chi_i^3(\mlc)$ there exists $\stru{\alpha_3,\mathbf{g}} \in \chi_i^3(\mlc) \setminus \chi_i^2(\mlc)$ with $\default_i(\alpha_3)$ at context $\mlc_3 \succeq_{-i} \mlc_{3b} \succ_{i}  \mlc$ such that $\mlc_{2b} \succ_i \mlc_{3b}$.

Case 1.2.1: If $\stru{\alpha_3,\mathbf{g}} \not\in \chi_i^1(\mlc)$ we are done. 

Case 1.2.2: Otherwise, $\stru{\alpha_3,\mathbf{g}} \in \chi_i^1(\mlc) \setminus \chi_i^2(\mlc)$. Note that this is the same situation as in case 1 except that $\default_i(\alpha_3)$ is at context $\mlc_{3b} \succ_i \mlc$ such that $\mlc_{1b} \succ_i \mlc_{2b} \succ_i \mlc_{3b}$. Since $\succ_i$ is a strict (partial) order and we only have finitely many contexts this can only occur finitely often. Since in all other cases below case 1 we have that $\chi_i^1(\mlc) > \chi_i^3(\mlc)$ we are done with case 1.

Case 2: If $\stru{\alpha_1,\ee} \in \chi_i^2(\mlc)$ we are in a similar situation as in case 1.2 the statement follows by analogous reasoning.
\end{proof}
\begin{proof}
As we have seen, \lstinline*LocPref($\mlc$,$rel$)* is transitive for each context $\mlc$ and relation $rel$. Thus their pareto combination is also transitive.
\end{proof}

As the domain of the semiring we choose $R = \{(S, \chi) \mid S \in \mathbb{N}^{B}, \casmap \text{ clashing assumption multiset-map}\}$. Here, we need $S$ to be a multiset and $\casmap$ to map to multisets of clashing assumptions for technical reasons (namely so that our semiring satisfies the distributive law). We generalize the definition of the local preference to multisets by using
 \begin{quote}
  $\chi_i^1(\mlc) > \chi_i^2(\mlc)$, if for every $\eta = \stru{\alpha_1,\ee}$ s.t.\ the multiplicity of $\eta$ in $\chi_i^1(\mlc)$ is greater than its multiplicity in $\chi_i^2(\mlc)$ with 
  $\default_i(\alpha_1)$ at a context $\mlc_1 \succeq_{-i} \mlc_{1b} \succ_{i}  \mlc$, 
  there exists an $\eta' = \stru{\alpha_2,\ff}$ s.t.\ the multiplicity of $\eta'$ in $\chi_i^2(\mlc)$ is greater than its multiplicity in $\chi_i^1(\mlc)$ with 
	$\default_i(\alpha_2)$ at context $\mlc_2 \succeq_{-i} \mlc_{2b} \succ_i \mlc$ 
	such that $\mlc_{1b} \succ_i \mlc_{2b}$.
\end{quote}
With this in mind, we can define the semiring $\mathcal{R}_{\rm one}(\CKB) = (R\cup\{\zero, \one\}, \srsplus, \srstimes, \zero, \one)$ by letting 
\begin{align*}
	a \srsplus b &= \left\{\begin{array}{ll}
 	a & \text{ if } a \succ_{\text{\lstinline*LocPref($rel$)*}} b\\
 	b & \text{ if } b \succ_{\text{\lstinline*LocPref($rel$)*}}a\\
 	{\rm lex{-}min}(a,b) & \text{ otherwise.}
 	\end{array}\right. 	\begin{array}{l}
  \zero \srsplus a := a =: a \srsplus \zero,\\
 \one \srsplus a := \one =: a \srsplus \one,\\
 \zero \srstimes a := \zero =: a \srstimes \zero,\\
 \one \srstimes a := a =: a \srstimes \one.
 \end{array}\\
	(S_1, \casmap^{(1)}) \srstimes (S_2, \casmap^{(2)}) &= (S_1 + S_2, \casmap^{(1)} + \casmap^{(2)})\\
\end{align*}
Here, ${\rm lex{-}min}(a,b)$ takes the lexicographical minimum of $a,b$ and the addition refers to pointwise multiset union, i.e., $(\casmap^{(1)} + \casmap^{(2)})(\mlc) = \casmap^{(1)}(\mlc) + \casmap^{(2)}(\mlc)$.

Now we can define the following weighted formula:
\begin{align*}%
\alpha_{\rm one} &= \alpha_1 \stimes \alpha_2 & \\
\alpha_1&= \btimes_{a \in B} (a \stimes (\{a\}, \emptyset)
\splus \neg a& \\
\alpha_2 &= \btimes_{c \in C}\; \btimes_{ \langle\alpha,e,i\rangle \in  pclash(c)} \ovr(\alpha, e, c, i) \stimes (\emptyset, \{c \mapsto \{\!\!\{\langle \alpha, e\rangle\}\!\!\}) \splus \neg \ovr(\alpha, e, c, i), 
\end{align*}
where $B$ is the Herbrand base and $pclash(c)= \{ \langle\alpha,e,i\rangle \mid$
$\langle \alpha, e\rangle$ is a possible clashing assumption for $c$ and $i\,\}$.
Intuitively, $\alpha_1$ collects the atoms that are true in the given interpretation and $\alpha_2$ builds the clashing assumption map, which is used to decide whether one interpretation is preferred over the other.
\begin{theorem}
$\mathcal{R}_{\rm one}(\CKB)$ is a semiring and the overall weight of $\mu = \langle PK(\CKB), \alpha_{\rm one}, \mathcal{R}_{\rm one}(\CKB)\rangle$ is $(I, \casmap)$, where $I$ is the minimum lexicographical preferred model of $\CKB$ and $\casmap$ its clashing assumption map or $\zero$ if there is no preferred model.
\end{theorem}
\begin{proof}
Associativity of $\srsplus$ follows from transitivity of \lstinline*LocPref($\mlc$,$rel$)* and the lexicographical order. Commutativity of $\srsplus$ is clear. $\mathbf{0}$ and $\mathbf{1}$ are identities and annihilators of $\srstimes, \srsplus$ by definition.
Associativity of $\srstimes$ is clear.

It remains to prove that multiplication distributes over addition. So let $A_i = (I_i, \chi_i)$ for $i = 1,2,3$. Then, in the expression
\begin{align*}
	A_1\srstimes (A_2 \srsplus A_3)
\end{align*}
Assume w.l.o.g. that $(A_2 \srsplus A_3)$ evaluates to $A_2$. If $A_2 >_{\text{\lstinline*LocPref($rel$)*}} A_3$ then there exists a context $\mlc$ such that 
$\chi_2(\mlc) >_{\text{\lstinline*LocPref($\mlc$,$rel$)*}} \chi_3(\mlc)$. Then it also holds that
$(\chi_1 + \chi_2)(\mlc) >_{\text{\lstinline*LocPref($\mlc$,$rel$)*}} (\chi_1 + \chi_3)(\mlc)$ and thus 
\begin{align*}
	A_1\srstimes (A_2 \srsplus A_3) = A_1\srstimes A_2 = A_1\srstimes A_2 \srsplus A_1\srstimes A_3.
\end{align*}
If $A_2 \not >_{\text{\lstinline*LocPref($rel$)*}} A_3$ this implies that $A_2$ is either equal to $A_3$ 
(in this case we are done) or
that $A_2$ is smaller lexicographically. 
In the latter case the sum $A_1\srstimes A_2$ is however also lexicographically smaller
than $A_1\srstimes A_3$ since we add $A_1$ both times.

Thus we have established that $\mathcal{R}_{\rm one}(\CKB)$ is a semiring. For each answer set $I$ of $PK(\CKB)$, we know that $I$ corresponds to a (least) CAS model. Thus, 
\[
	\llbracket \alpha_{\rm one} \rrbracket_{\mathcal{R}_{\rm one}}(I) = (\IC, \casmap),
\]
where $\IC \in \{0,1\}^{B}$ and $\casmap$ only maps to multisets that can be interpreted as sets (i.e. each of their elements has at most multiplicity one).
The lexicographical minimum CKR model $\stru{\IC^*, \casmap^*}$ satisfies that $(\IC, \casmap) \srsplus (\IC^*, \casmap^*) = (\IC^*, \casmap^*)$ for all $(\IC, \casmap)$ that are the semantics of $\alpha_{\rm one}$  w.r.t.\ some answer set of $PK(\CKB)$. Therefore, if there exists a CKR model, the overall weight is $(\IC^*, \casmap^*)$. Otherwise, it is $\mathbf{0}$.
\end{proof}

We continue with the $\mathcal{R}_{\rm all}$ semiring. Again, we need some additional lemma
\begin{lemma}
\label{lem:local_opt}
Let $\CKB$ be a single relational sCKR without $\eval$
expressions. Then a CAS model $(\{\Ic\}_{\mlc \in
  \N}, \casmap)$ is a CKR model iff no CAS model
$(\{\mathcal{I}'_{\mlc}\}_{\mlc \in \N}, \casmap')$ and $\mlc \in \N$ exist such that $\casmap'(\mlc) > \casmap(\mlc)$.
\end{lemma}
Therefore, we can take the locally optimal models $\Ic$ for each context $\mlc$ and obtain the global optimal models as arbitrary combinations of locally preferred models. 

In the following, we let $\mathcal{D}$ be the Herbrand base. 

Using this notation, we define the semiring $\mathcal{R}_{\mlc} = (R_{\mlc}, \srsplus_{\mlc}, \srstimes_{\mlc}, \srzero, \srone)$ that collects all locally optimal models $\Ic$. Here,
\begin{align*}
	R_{\mlc} &= \{ opt(B) \mid A \subseteq 2^{\mathcal{D}}, B = \{(S, \chi) \mid S \in A, \chi \text{ multiset of clashing assumptions}\} \}\\
	A \srsplus_{\mlc} B &=  opt_{\mlc}(A\cup B)\\
	A \srstimes_{\mlc} B &= opt_{\mlc}(\{(S_1 \cup S_2, \chi_1 + \chi_2) \mid (S_1, \chi_1) \in A, (S_2, \chi_2) \in B\}\\
	\srzero &= \emptyset \\
	\srone &= \{ (\emptyset, \{\!\!\{\}\!\!\})\}\\
	opt_{\mlc}(A) &= \{(S, \chi) \in A \mid \forall (S',\chi') \in A: \neg(\chi'(\mlc) > \chi(\mlc))\}
\end{align*}
We again have to use multisets for $\chi$ instead of sets. This is necessary because otherwise multiplication and addition do not satisfy the distributive law.

Then, we can define the measure $\mu_{\mlc} = \langle PK(\CKB), \alpha_{\rm all}, \mathcal{R}_{\mlc}\rangle$, where 
\begin{align*}
	\alpha_{\rm all} &= \alpha_1 \stimes \alpha_2\\
	\alpha_1 &= \btimes_{d \in \mathcal{D}} d \stimes \{(\{d\}, \{\!\!\{\}\!\!\})\} \splus \neg d\\
	\alpha_2 &= \btimes_{ \langle\alpha,e,i\rangle \in  pclash(c)} \ovr(\alpha, e, c) \stimes \{(\emptyset, \{c \mapsto \{\!\!\{\langle \alpha, e\rangle\}\!\!\})\} \splus \neg \ovr(\alpha, e, c).
\end{align*}
where $pclash(c)$ is the set of all possible clashing assumptions $\langle\alpha,e\rangle$ for
$c$. We obtain

\begin{theorem}
$\mathcal{R}_{\mlc}$ is a semiring and the overall weight $\mu_{\mlc}(PK(\CKB))$ is equal to the set containing for each locally optimal interpretation $\Ic$ of $\CKB$ the pair $(\Ic, \chi_{\Ic})$, where $\chi_{\Ic}$ is the unique multiset containing each justified clashing assumption of $\Ic$ once.
\end{theorem}

We take $\mathcal{R}_{\rm all}$ to be the crossproduct semiring $(\mathcal{R}_{\mlc})_{\mlc \in \N}$
defined by
\begin{align*}
(\mathcal{R}_{\mlc})_{\mlc \in \N} = &\ ((R_{\mlc})_{\mlc \in
        \N}, \srsplus, \srstimes, (\emptyset)_{\mlc \in \N},
  (\{(\emptyset, \{\!\!\{\}\!\!\})\})_{\mlc \in \N}), \text{ where}\\[2pt]
& (A_{\mlc})_{\mlc \in \N}  \odot (B_{\mlc})_{\mlc \in \N} =
  (A_{\mlc} \odot_{\mlc} B_{\mlc})_{\mlc \in \N}, \text{ for } \odot
        \in \{ \srsplus, \srstimes \}
\end{align*}
Using it, we can obtain the locally optimal interpretations for each context as the crossproduct of measures $\mu^{*} = (\mu_{\mlc})_{\mlc \in \N}$ which is a measure over the crossproduct semiring $(\mathcal{R}_{\mlc})_{\mlc \in \N}$. As we have shown in Lemma~\ref{lem:local_opt}, this gives us all the preferred models. Namely, let $\mu^{*}(PK(\CKB)) = (A_{\mlc})_{\mlc \in \N}$, then 
$
\{(\Ic)_{\mlc \in \N} \mid (\Ic,\casmap(\mlc)) \in A_{\mlc}\}
$
is the set of preferred models.
\begin{example}
The sCKR $\CKB$ defined in Example~\ref{ex:coverage-only} has five contexts $\mlc_{world},$ $\mlc_{branch1},$ $\mlc_{branch2},$ $\mlc_{local1}, $ and $\mlc_{local2}$. Therefore, the measure $\mu^{*}$ is a crossproduct of the five measures $\mu_{\mlc_{world}}, \mu_{\mlc_{branch1}}, \mu_{\mlc_{branch2}}, \mu_{\mlc_{local1}}$ and $\mu_{\mlc_{local2}}$. Their overall weight is given by 
\begin{alignat*}
\mu_{\mlc_{world}}(PK(\CKB)) = \mu_{\mlc_{branch1}}(PK(\CKB)) = \mu_{\mlc_{branch2}}(PK(\CKB)) = \mu_{\mlc_{local2}}(PK(\CKB)) =\{(\emptyset, \{\!\!\{\}\!\!\}\}\\
\mu_{\mlc_{local1}}(PK(\CKB)) = \{(\{S(i), M(i)\}, \{\!\!\{\ovr(S \subs R, i, \mlc_{branch2}), \ovr(S \subs E, i, \mlc_{world})\}\!\!\})\}
\end{alignat*}
Accordingly, there is exactly one preferred model $(\Ic)_{\mlc \in \N}$, where 
\begin{align*}
\mathcal{I}_{\mlc_{world}} = \mathcal{I}_{\mlc_{branch1}} = \mathcal{I}_{\mlc_{branch2}} = \mathcal{I}_{\mlc_{local2}} =& \emptyset & 
\mathcal{I}_{\mlc_{local1}} &= \{S(i), M(i)\}
\end{align*}
\end{example}

\begin{theorem}
Let $\CKB$ be a single-relational, \eval-free sCKR, then $\mathcal{R}_{\rm all}$ is a semiring and the overall weight of $\mu_{\rm all} = \langle PK(\CKB), \alpha_{\rm all}, \mathcal{R}_{\rm all}(\CKB)\rangle$ is $(A_{\mlc})_{\mlc \in \N}$ and
the set of preferred models corresponds to
$\{(\Ic)_{\mlc \in \N} \mid \text{ for each }c \in \N: (\Ic,\casmap(\mlc)) \in A_{\mlc}\}$.
\end{theorem}
\begin{proof}
The reasoning that $\mathcal{R}_{\rm all}$ is a semiring is along the same lines as that for $\mathcal{R}_{\rm one}$. The fact that the result is as desired can be clearly seen during the construction of the semiring.
\end{proof}

\label{lastpage}
\end{document}